\pdfoutput=1
\documentclass{article}
\usepackage{preprint,times}

\usepackage{amsmath,amsfonts,bm}

\def\eqref#1{equation~\ref{#1}}

\def\1{\bm{1}}

\DeclareMathAlphabet{\mathsfit}{\encodingdefault}{\sfdefault}{m}{sl}
\SetMathAlphabet{\mathsfit}{bold}{\encodingdefault}{\sfdefault}{bx}{n}

\usepackage[utf8]{inputenc}
\usepackage[T1]{fontenc}
\usepackage{amsmath}
\usepackage{amssymb}
\usepackage{amsthm}

\usepackage{algorithm}
\usepackage{algpseudocode}
\let\REQUIRE\Require
\let\ENSURE\Ensure
\let\STATE\State
\let\FOR\For
\let\ENDFOR\EndFor
\let\IF\If
\let\ELSE\Else
\let\ENDIF\EndIf
\usepackage{array,longtable}
\usepackage{booktabs}
\usepackage{enumitem}
\usepackage{float}
\usepackage{flafter}
\usepackage{placeins}
\usepackage{graphicx}
\usepackage{microtype}
\usepackage{multirow}
\usepackage{tabularx}
\usepackage{tikz}
\usepackage{xcolor}
\usepackage{url}
\usepackage{hyperref}
\hypersetup{hidelinks}
\graphicspath{{Figs/}}
\usetikzlibrary{arrows.meta,positioning,fit}
\definecolor{stableblue}{RGB}{26,92,164}
\definecolor{stableorange}{RGB}{224,116,38}
\definecolor{stablegreen}{RGB}{34,139,94}

\title{\raggedright TurboPairFormer: Fast and Stable Protein Folding Model Training
with an Optimized Triangle Attention Kernel}
\author{
  \hspace{-1mm}\textbf{Yide Ran}$^{1}$\textnormal{,} 
  \textbf{Chelsea Lowman}$^{2}$\textnormal{,}  
  \textbf{Jan Domański}$^{3}$\textnormal{,}  
  \textbf{David Hartmann}$^{2}$\textnormal{,}  
  \textbf{Jenke Scheen}$^{3}$\textnormal{,}  \\
  \textbf{Jennifer Wei}$^{3}$, 
  \textbf{Chuan Li}$^{2}$, 
  \textbf{Jianwen Xie}$^{\dagger2}$, 
  \textbf{Zhaozhuo Xu}$^{\dagger1}$ \\
  $^1$Stevens Institute of Technology, 
  $^2$Lambda, 
  $^3$Open Molecular Software Foundation
}
\newif\ifdraftnotes
\draftnotestrue
\begin{document}
\maketitle

\renewcommand*\thefootnote{}
\footnotetext{    
    $^{\dagger}$Equal advising \hspace{0.2cm}Website: \url{https://pypi.org/project/turbopairformer}
}

\begin{abstract}
\looseness=-1
Triangular attention is a core computation in AlphaFold3-style biomolecular models, with computational cost that grows cubically with token count. FlashAttention reduces attention memory traffic through IO-aware tiling and backward recomputation. However, the shared pair bias in triangular attention adds a gradient reduction across attention slices to the usual reductions over queries and keys. The open-source backends we examine handle these reductions through repeated probability recomputation, floating-point atomics, or full score-gradient storage. Separately, computing the softmax backward correction from BF16-rounded forward outputs incurs a loss of numerical precision. We present TurboPairFormer, a triangular attention implementation for NVIDIA Hopper GPUs that addresses these two issues. Our key-tile-parallel backward algorithm recomputes each probability tile once and reuses it for the query, key, value, and pair-bias gradients, using ordered partial reductions to achieve deterministic accumulation without floating-point atomics or full score-gradient storage. Our output-residual compensation retains a BF16 approximation of the output-rounding residual and uses it to compute the backward correction more accurately in FP32, without changing the BF16 output. With BF16 inputs at crop sizes 384, 640, and 768 and head dimensions 16 and 32, TurboPairFormer achieves the lowest mean query, key, and pair-bias gradient RMSE against an FP64 reference among eight evaluated implementations; residual compensation mitigates the precision loss from output rounding, reducing these RMSE values by 28--47\% in controlled ablations. All four gradients are bitwise identical across five repeated calls in all 600 input cases under fixed execution conditions. Integrated into OpenFold3 together with our triangle multiplication kernels, TurboPairFormer achieves the lowest GPU computation time per optimizer step among eight evaluated backend configurations on 16 H100 GPUs, with speedups of $1.73\times$ over OpenFold3's Triton backend and $1.13\times$ over cuEquivariance at crop size 768.
\end{abstract}
\section{Introduction}
\label{sec:intro}

AlphaFold3 predicts the joint structure of complexes containing proteins, nucleic acids, and small molecules~\citep{abramson2024alphafold3}. OpenFold3 provides an open-source implementation of AlphaFold3 with support for training and fine-tuning~\citep{openfold3preview}. Both models use one token per standard amino acid or nucleotide residue and one token per heavy atom for other components. They maintain a pair representation, with one learned vector for each pair of tokens, and repeatedly update it through triangular attention and triangle multiplication. In triangular attention, fixing one token index, which we call the outer index, selects the pair vectors between that token and every token in the crop. These vectors form an attention slice with one query and one key per token, as in standard attention, and all slices share the same pair bias. There is one slice per token, and the cost of each slice grows with the square of the token count. The cost of triangular attention therefore grows with the cube of the token count: doubling the training crop makes it eight times as expensive.

IO-aware kernels such as FlashAttention fuse attention into tiled kernels and recompute probabilities in the backward pass instead of storing them~\citep{dao2022flashattention}; their successors improve parallelism and overlap data movement with computation on Hopper GPUs~\citep{dao2023flashattention2,shah2024flashattention3}. However, extending these kernels to triangular attention requires an additional reduction of the shared pair-bias gradient across slices. The backward pass thus has three reduction directions (Figure~\ref{fig:design_comparison}a): the query gradient sums over keys, the key and value gradients over queries, and the pair-bias gradient over outer indices. All four gradients reuse the same probability tiles, but their reductions favor different work partitionings. With one recomputation, they share one work partitioning, so any gradient whose reduction axis is split across thread blocks requires a cross-block reduction.

The open-source triangular-attention kernels we inspect make different compromises to handle this reduction (Figure~\ref{fig:design_comparison}b): TriFast~\citep{atkins2025trifast} and Protenix~\citep{protenix2025} recompute the probabilities three times, DeepSpeed-Evo~\citep{song2023deepspeed4science} accumulates the pair-bias gradient with floating-point atomics whose order is not fixed, and ScaleFold~\citep{zhu2024scalefold} stores the full FP32 score gradient. FlashAttention-3's deterministic mode orders in-place query-gradient updates with semaphores, and DASH improves its scheduling~\citep{shah2024flashattention3,qiang2026dash}, but neither computes a bias gradient.

\begin{figure}[H]
\centering
\includegraphics[width=0.95\linewidth]{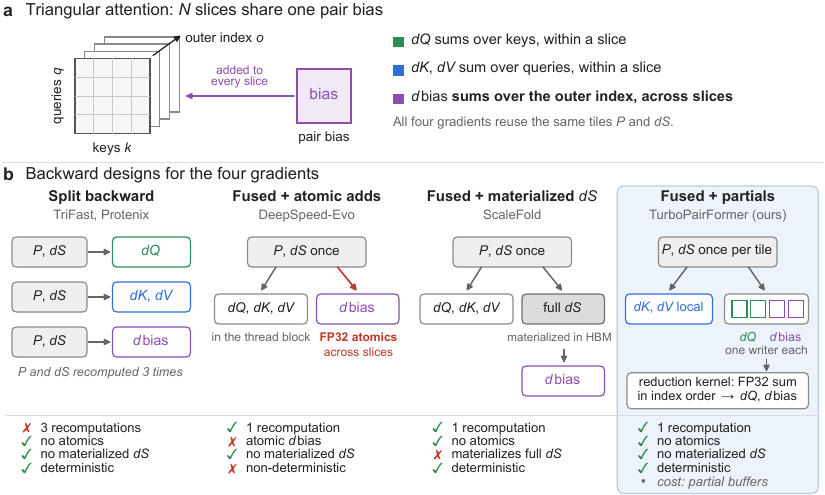}
\caption{\textbf{Triangular attention and backward design comparison.} (a)~Triangular attention consists of $N$ slices sharing one pair bias. Its backward pass involves three reduction axes: $dQ$ over keys, $dK$ and $dV$ over queries, and $d\,\mathrm{bias}$ across outer slices. All four gradients share the same intermediates ($P$ and $dS$).
(b)~Prior backends either recompute probabilities 3$\times$, use non-deterministic atomics, or store the full score gradient $dS$, whose size is cubic in $N$. TurboPairFormer evaluates $P$ and $dS$ only once: it keeps $dK$ and $dV$ in local registers, writes $dQ$ and $d\,\mathrm{bias}$ into independent partial buffers, and reduces them in a deterministic FP32 order.}
\label{fig:design_comparison}
\vspace{-1em}
\end{figure}

Determinism does not guarantee accuracy. In the softmax backward pass, all score gradients in a query row depend on the same correction $\delta$, the dot product of the output gradient and the forward output for that row~\citep{dao2022flashattention}. Kernels usually compute $\delta$ from the saved BF16 output, whose rounding error can shift these score gradients in the same direction. A deterministic kernel reproduces this error on every run. The forward pass matters as well: with 64-row query tiles and head dimension 32, a bias tile holds as many bytes as the key and value tiles combined, and recycling and checkpointing run two to five forward passes per backward pass. Our challenge is therefore threefold: to reuse each recomputed probability tile for all four gradients with fixed accumulation orders and sufficient parallelism, to compute $\delta$ accurately, and to stage the shared bias efficiently in the forward pass.

We propose TurboPairFormer, a triangular-attention kernel for Hopper GPUs:
\begin{enumerate}[leftmargin=*]
\item \textbf{Work partitioning for three reduction directions:} We analyze how the shared pair bias trades recomputation against parallelism and storage (Section~\ref{sec:problem} and Appendix~\ref{tpfapp:determinism}), and assign each thread block a key tile and a group of outer indices, which keeps the key and value gradients local and leaves two cross-block reductions.
\item \textbf{Single recomputation with ordered partial reductions:} Like DeepSpeed-Evo and ScaleFold, TurboPairFormer recomputes each probability tile once for all four gradients, but without floating-point atomics or full score-gradient storage. Query and pair-bias partials have unique writers and are reduced in fixed FP32 orders, without blocks waiting for one another. Gradient bits do not depend on thread-block completion order or on supported workspace schedules; the cost is a query-partial buffer cubic in crop size (Appendix~\ref{tpfapp:memory}).
\item \textbf{Output-residual compensation:} The forward epilogue saves a BF16 approximation of the output-rounding residual, which the backward pass adds to the BF16 output in FP32 when computing $\delta$~\citep{qiu2026lowprecision,henry2019leveraging}. This lowers the mean RMSE of the query, key, and pair-bias gradients by 28--47\% in controlled ablations, without changing the BF16 output or adding a recomputation; the cost is one extra output-sized BF16 tensor.
\item \textbf{Forward staging of the shared bias:} The forward kernel loads each bias tile with the Tensor Memory Accelerator alongside the key and value tiles, reads it directly in the score-accumulator register layout, prefetches the key mask, and evaluates probabilities in 16-key chunks just before value aggregation, without changing the floating-point operations or their order.
\end{enumerate}

We evaluate TurboPairFormer on H100 GPUs at crop sizes of 384, 640, and 768 and head dimensions $D=16$ and $32$. Compared with the other methods evaluated in this paper, TurboPairFormer produces more accurate query, key, and pair-bias gradients, as measured by mean RMSE against an FP64 reference. All four gradients are bitwise identical across five repeated calls in all 600 fixed-input cases (Appendix~\ref{sec:tpf-supp-replay}), whereas cuEquivariance, OpenFold3 Triton, and DeepSpeed-Evo pass none; the fastest non-deterministic variants of our kernel take 2.8--4.1\% less backward time at $D=32$ and 11.0--13.7\% less at $D=16$ (Appendix~\ref{tpfapp:determinism}). One TurboPairFormer attention call takes 35.6--45.6\% less forward-plus-backward time than cuEquivariance~\citep{nvidia2026cuequivariance} and ranges from 7.4\% faster to 3.6\% slower than Protenix (Appendix~\ref{tpfapp:operator-attention}); the forward changes cut forward time by 25.5--27.5\% at $D=32$ (Appendix~\ref{tpfapp:forward-ablation}). Combined with our triangle multiplication kernels, TurboPairFormer has the lowest GPU computation time per step among eight OpenFold3 backend configurations on 16 H100 GPUs, $1.73\times$ faster than OpenFold3 Triton and $1.13\times$ faster than cuEquivariance at crop 768. In paired 400-update continuations from a shared step-1000 checkpoint, the losses of runs using an earlier build of our kernels remain finite, whereas OpenFold3 Triton produces NaN losses within 10 updates (Section~\ref{sec:stability}).

\section{TurboPairFormer}
\label{sec:method}

TurboPairFormer implements triangular attention and triangle multiplication for OpenFold3 pair updates (TriMul details in Appendix~\ref{tpfapp:trimul}). The attention kernel combines one probability recomputation, fixed accumulation orders, and output-residual compensation. Figure~\ref{fig:partition} and Algorithm~\ref{alg:tpf_backward} summarize the backward pass; Table~\ref{tpfapp:tab:notation} defines the notation.

\subsection{Triangular Attention and Its Reduction Axes}
\label{sec:problem}

\looseness=-1
OpenFold3 learns features $Z$ for each ordered token pair. Standard amino acids and nucleotides contribute one token each; other components use one token per heavy atom~\citep{abramson2024alphafold3}. A crop selects $N$ tokens, possibly across chains. Starting-node triangular attention fixes a token index $o$, called the \emph{outer index}, and takes row $o$ of $Z$ as one slice. Within the slice, the token indices $q$ and $k$ mark the query and key positions (worked example in Figure~\ref{tpfapp:fig:worked-example}, Appendix~\ref{tpfapp:worked-example}). The outer index identifies a token, not a chain or geometric plane. Ending-node attention uses the same core after permuting the pair axes (Appendix~\ref{tpfapp:outer}).

For one slice and one head, omitting the batch index, linear projections of the slice give query, key, and value matrices $Q,K,V\in\mathbb R^{N\times D}$ with one row per pair, where $D$ is the head dimension: row $q$ of $Q$ comes from the pair $(o,q)$, and row $k$ of $K$ and $V$ from the pair $(o,k)$. The pair bias $\mathrm{bias}\in\mathbb R^{N\times N}$, whose entry $(q,k)$ is projected from the pair $(q,k)$, is shared by all $N$ slices. With additive mask $M$, the attention core computes
\begin{equation}
S=\frac{QK^\top}{\sqrt{D}}+\mathrm{bias}+M,\qquad
P=\operatorname{softmax}(S),\qquad
O=PV,
\label{eq:tpf_forward}
\end{equation}
where the softmax acts on each row. Given the incoming gradient $dO$, the softmax backward forms the probability gradient $dP$ and the score gradient $dS$, which with $P$ give $dQ$, $dK$, and $dV$:
\begin{equation}
\begin{aligned}
dP&=dO\,V^\top, & dS&=P\circ(dP-\delta),\\
dQ&=\frac{dS\,K}{\sqrt{D}}, & dK&=\frac{dS^\top Q}{\sqrt{D}},\\
dV&=P^\top dO, & \delta&=\operatorname{rowsum}(dO\circ O),
\end{aligned}
\label{eq:dbias}
\end{equation}
where $\delta$ is the row correction, one value per query row. Because every slice adds the same bias, the fourth gradient, $d\,\mathrm{bias}$, is the sum of $dS$ over all $N$ slices. These gradients reuse the same intermediates but sum along different axes: keys for $dQ$, queries for $dK/dV$, and outer indices for $d\,\mathrm{bias}$. This determines which results a thread block can finish locally and which require a later reduction. Appendix~\ref{tpfapp:gradients} derives the identities.

\subsection{One Recomputation for Four Gradients}
\label{sec:single}

Each thread block, or cooperative thread array (CTA), owns one key tile and a group of four outer indices, called an outer group, for one example and head. It traverses query tiles in ascending order, using 64 keys and 32 query rows per tile. Different key tiles provide independent work. Within each owned slice, $dK/dV$ accumulators stay in FP32 registers throughout the traversal (Figure~\ref{fig:partition}). For each query tile, the thread block recomputes $P=\exp(S-\mathrm{LSE})$ from the scores and saved row log-sum-exp. It forms $dP$ and $dS$ as in Equation~(\ref{eq:dbias}), with the compensated row correction $\widehat\delta$ of Section~\ref{sec:delta} in place of $\delta$. The tile $P$ supplies $dV$, while $dS$ supplies $dK$, $dQ$, and the pair-bias gradient. Shared memory makes these intermediates available to their consumers before the next tile. Each thread block writes its query contributions as partials indexed by key tile and its bias contributions as partials indexed by outer group. Every partial has one writer. Separate reduction kernels sum these BF16 partials in FP32, without recomputing probabilities or storing the full score gradient. Proposition~\ref{tpfapp:prop:partition} establishes the partitioning. Algorithm~\ref{alg:tpf_backward} omits batch and head indices; inside the slice loop, $Q$, $K$, $V$, $dO$, $M$, $\mathrm{LSE}$, $\widehat\delta$, $dK$, and $dV$ denote the current tiles of the current slice. Row statistics and the key mask broadcast along the other tile axis; fully masked rows and invalid outer positions have $P=0$. All matrix products use BF16 operands and FP32 accumulation, so $P$ and $dS$ are rounded when consumed by a matrix product; the bias sum instead uses FP32 $dS$. The main kernel uses two warpgroups (256 threads), WGMMA for $S$ and $dP$, and \texttt{mma.sync} for the gradient products\footnote{NVIDIA \href{https://docs.nvidia.com/cuda/archive/12.9.1/parallel-thread-execution/index.html}{PTX ISA 8.8}, CUDA 12.9.1 archive (2025): \href{https://docs.nvidia.com/cuda/archive/12.9.1/parallel-thread-execution/index.html\#asynchronous-warpgroup-level-matrix-instructions-wgmma-mma}{\texttt{wgmma.mma\_async}} and \href{https://docs.nvidia.com/cuda/archive/12.9.1/parallel-thread-execution/index.html\#warp-level-matrix-instructions-mma}{\texttt{mma.sync}}.}. Appendix~\ref{tpfapp:algorithms} gives the thread-block view with explicit slice and tile indices, including the fixed addition tree over the four slices, synchronization, and instruction order.

\begin{figure}[!t]
\centering
\includegraphics[width=\linewidth]{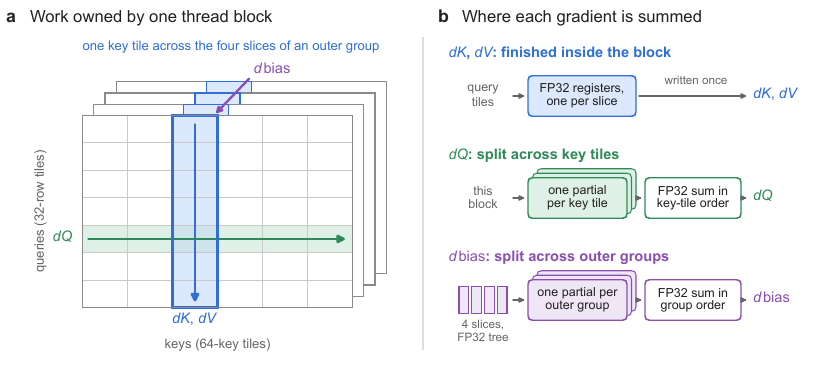}
\caption{\textbf{Work partitioning of the backward kernel.} (a)~A thread block owns one key tile in each of the four slices of an outer group and visits every query tile. (b)~$dK$ and $dV$ finish inside the block; $dQ$ and $d\,\mathrm{bias}$ leave as partials with one writer each, which separate kernels add in a fixed order.}
\label{fig:partition}
\end{figure}

\begin{algorithm}[!htbp]
\caption{\textsc{TurboPairFormer} backward pass}
\label{alg:tpf_backward}
\small
\algrenewcommand\algorithmicindent{1.0em}
\begin{algorithmic}[1]
\Require For each slice: $Q,K,V,dO\in\mathbb R^{N\times D}$, mask $M$, saved $\mathrm{LSE}\in\mathbb R^{N}$, and the saved output and residual (Section~\ref{sec:delta}). The shared $\mathrm{bias}\in\mathbb R^{N\times N}$. All in HBM.
\State In a preprocessing kernel, compute the compensated row correction $\widehat\delta\in\mathbb R^{N}$ of every slice in FP32.
\State Divide the keys into tiles of 64, the queries into tiles of 32, and the outer indices into outer groups of four.
\For{each key tile and outer group, in parallel}
  \State Load $K$, $V$, and $M$ of the key tile for the four slices from HBM to on-chip memory.
  \State On chip, initialize $dK=dV=(0)\in\mathbb R^{64\times D}$ in FP32 for each of the four slices.
  \For{each query tile, in ascending order}
    \State Load $Q$, $dO$, $\mathrm{LSE}$, and $\widehat\delta$ for the four slices, and the bias tile, to on-chip memory.
    \For{each of the four slices}
      \State On chip, compute $S=QK^\top/\sqrt D+\mathrm{bias}+M$ and recompute $P=\exp(S-\mathrm{LSE})$.
      \State On chip, compute $dP=dO\,V^\top$ and $dS=P\circ(dP-\widehat\delta)\in\mathbb R^{32\times64}$.
      \State On chip, update $dV\gets dV+P^\top dO$ and $dK\gets dK+dS^\top Q$.
      \State Write the $dQ$ partial $dS\,K\in\mathbb R^{32\times D}$ of this key tile to HBM.
    \EndFor
    \State Add the four $dS$ in a fixed FP32 order, and write the $d\,\mathrm{bias}$ partial of this outer group to HBM.
  \EndFor
  \State Write $dK/\sqrt D$ and $dV$ of the four slices to HBM.
\EndFor
\State In a reduction kernel, add the $dQ$ partials in ascending key-tile order in FP32, and scale by $1/\sqrt D$.
\State In a reduction kernel, add the $d\,\mathrm{bias}$ partials in ascending outer-group order in FP32.
\State \Return $dQ,dK,dV,d\,\mathrm{bias}$.
\end{algorithmic}
\end{algorithm}

\subsection{Deterministic Reductions and Workspace Scheduling}
\label{sec:leaves}
\label{sec:execution}

Partitioning work across key tiles and outer groups exposes parallelism, but distributes contributions to $dQ$ and the shared pair-bias gradient across thread blocks. We fix their floating-point addition order to make these reductions reproducible. TurboPairFormer separates the production of these contributions from their reduction: each block writes its partials to unique addresses, and a later kernel consumes them in a prescribed order. The addition sequence is therefore independent of block completion order. For each $dQ$ element, the reducer adds BF16 partials sequentially in ascending key-tile order, using an FP32 accumulator. The pair-bias reduction follows the same principle along the outer axis: each outer group first forms one BF16 partial through a fixed FP32 addition tree, and the reducer adds these partials in ascending outer-group order. The $dK/dV$ accumulations remain local to their owning block and follow ascending query-tile order. Final scaling and BF16 output conversion occur after the respective reductions. These fixed paths need no global floating-point atomics. The workspace planner changes how much work is resident in each launch while preserving these arithmetic paths. It partitions examples, query rows, and outer groups into consecutive chunks. After reducing a chunk of bias partials, the next launch reloads the FP32 running sum and continues with the next outer group. If query traversal spans launches, $dK/dV$ likewise resume their FP32 accumulators at the next query tile. A launch boundary thus inserts only a store and reload; it introduces neither a new subtotal nor an extra rounding step. For fixed inputs, binary, device, tile sizes, and outer grouping, supported workspace schedules consequently produce bitwise-identical gradients. Appendix~\ref{tpfapp:invariance} proves this property; Algorithm~\ref{tpfapp:alg:leaf} specifies the launch boundaries. In FlashAttention-3's deterministic backward, semaphores order in-place $dQ$ updates~\citep{shah2024flashattention3,qiang2026dash}. TurboPairFormer orders the later reductions, allowing key-tile blocks to write their partials independently. This approach also covers the shared pair-bias gradient. Appendix~\ref{tpfapp:determinism} compares these designs with FlashAttention-2 and DASH. \label{sec:workspace}
The cost is storage for the partials and an additional pass to read them. Reusing the pair-bias buffers limits their resident size, but the current implementation retains all $dQ$ partials, one per key tile, so their size grows with the cube of the crop size $N$. Workspace scheduling does not remove this cubic allocation. Whether it increases peak training memory depends on which other tensors are live during the reduction; Appendix~\ref{tpfapp:memory} gives the allocation sizes and lifetime analysis.

\subsection{Preserving Precision with an Output Residual}
\label{sec:delta}

The row correction $\delta$ in Equation~(\ref{eq:dbias}) comes from the forward output without another probability traversal. Computing it from the saved BF16 output loses information discarded during output conversion. Following analyses of attention precision and multiple-component BF16 representations~\citep{qiu2026lowprecision,henry2019leveraging}, we save a BF16 approximation to the difference between the FP32 forward result $O_{32}$ and the public output $O_{16}$:
\begin{equation}
E_{16}=\operatorname{RN}_{16}(O_{32}-O_{16}),\qquad
\widehat\delta\approx\operatorname{rowsum}\bigl(dO\circ(O_{16}+E_{16})\bigr).
\label{eq:tpf_residual}
\end{equation}
Here $\operatorname{RN}_{16}$ rounds to BF16. The subtraction, reconstruction, and dot product execute in FP32; the approximation sign accounts for floating-point evaluation. In our controlled ablations, using $\widehat\delta$ in the shared $dS$ reduces the mean RMSE of $dQ$, $dK$, and the pair-bias gradient, while leaving the public output and $dV$ bitwise unchanged (Section~\ref{sec:accuracy_determinism}). The cost is one output-sized BF16 tensor and its memory traffic, without additional probability recomputation. Appendix~\ref{tpfapp:residual} specifies the evaluation and analyzes its precision.

\paragraph{Forward pipeline.}
\label{sec:forward}
The forward kernel employs tiled online softmax~\citep{milakov2018onlinesoftmax}, using double-buffered shared memory to overlap Tensor Memory Accelerator\footnote{NVIDIA \href{https://docs.nvidia.com/cuda/archive/12.9.1/hopper-tuning-guide/index.html\#tensor-memory-accelerator}{Hopper Tuning Guide, Section 1.4.1.2: Tensor Memory Accelerator}, CUDA 12.9.1 archive (2025).} loads of keys, values, and pair bias with Tensor Core compute. Probability evaluation is interleaved with value aggregation, and the epilogue saves the rounding residual alongside the BF16 output (Algorithm~\ref{tpfapp:alg:forward}; details in Appendix~\ref{tpfapp:forward}).

\FloatBarrier
\section{Experiments}
\label{sec:eval}

We evaluate TurboPairFormer through training performance, backward numerical accuracy against an FP64 reference, training stability, and fine-tuning validation; supplementary hardware measurements, additional accuracy results, component ablations, deterministic replay, and validation details appear in the appendix.

\subsection{Training Performance}
\label{sec:performance}

We compare TurboPairFormer, Protenix, cuEquivariance, TriFast, MegaFold, DeepSpeed-Evo, ScaleFold, and OpenFold3 Triton at crops 384, 640, and 768 on two nodes with eight NVIDIA H100 GPUs each and 16 samples per optimizer step. Each backend runs one 200-step training trajectory per crop. We measure GPU computation time with CUDA events, taking the maximum across ranks at each step and averaging over steps 21--200. Timings exclude data loading, host-to-device transfer, and EMA updates. \textbf{TurboPairFormer achieves the lowest measured mean GPU computation time at all three crops}: 5.26, 9.40, and 12.58 seconds per optimizer step, respectively (Figure~\ref{fig:tpf-training-performance}). These times are 17.4--42.1\% lower than OpenFold3 Triton's and 6.1--11.7\% lower than cuEquivariance's. Appendices~\ref{tpfapp:h200-training}, \ref{tpfapp:trimul}, and~\ref{sec:tpf-supp-stability} report supplementary training measurements on eight H200 GPUs, component ablations, and training-loss comparisons, respectively.

\begin{figure}[!htbp]
\centering
\includegraphics[width=0.94\linewidth]{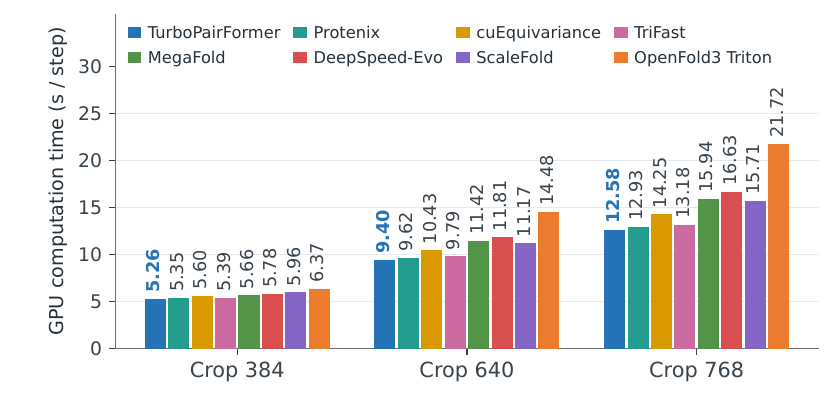}
\caption{\textbf{OpenFold3 training performance on 16$\times$H100 GPUs.} Mean GPU computation time per optimizer step on two nodes with eight H100 GPUs each and 16 samples per step. Statistics cover steps 21--200; data loading, host-to-device transfer, and EMA updates are excluded. Lower is better. TurboPairFormer has the lowest measured mean at all three crop sizes.}
\label{fig:tpf-training-performance}
\label{fig:tpf-h100-training}
\end{figure}

\subsection{Backward Numerical Accuracy}
\label{sec:accuracy_determinism}

On one H100, we compare eight attention implementations at crops $N\in\{384,640,768\}$ and head dimensions $D\in\{16,32\}$, using ten seeds per shape. All methods receive the same BF16 inputs and incoming gradient, with zero pair bias and mask. Figure~\ref{fig:tpf-backward-accuracy} reports the mean and sample standard deviation of absolute RMSE for the four public BF16 gradients against an FP64 reference evaluated on the same quantized inputs. TurboPairFormer has the lowest mean RMSE for $dQ$, $dK$, and $d\,\mathrm{bias}$ in all six shapes, reducing errors relative to cuEquivariance by 31.4--44.4\%, 28.7--42.4\%, and 16.9--37.9\%, respectively. Its $dV$ error is essentially tied with that of cuEquivariance. Enabling residual compensation lowers the $dQ$, $dK$, and $d\,\mathrm{bias}$ errors in every paired case, by 36.3--47.1\%, 28.7--42.4\%, and 28.3--43.1\% in mean RMSE, while the output and $dV$ stay bitwise identical (Figure~\ref{tpfapp:fig:residual-on-off}; Appendix~\ref{sec:tpf-supp-residual} adds captured-input checks). Appendix~\ref{sec:tpf-supp-replay} reports five-repeat gradient determinism across 600 fixed inputs.

\begin{figure}[!htbp]
\centering
\includegraphics[width=0.94\linewidth]{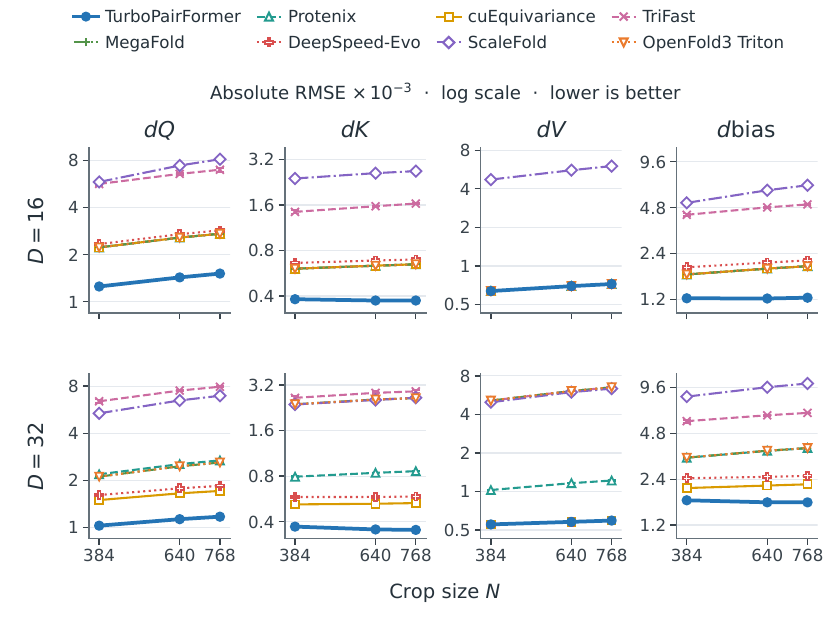}
\caption{\textbf{Attention backward accuracy against FP64.} Rows show $D=16$ and $D=32$; columns show $dQ$, $dK$, $dV$, and $d\,\mathrm{bias}$. Points and error bars report mean absolute RMSE and sample standard deviation over ten seeds. Logarithmic axes show RMSE in units of $10^{-3}$. Implementation variants and additional results are described in Appendix~\ref{sec:tpf-supp-protocol}.}
\label{fig:tpf-backward-accuracy}
\end{figure}

\subsection{Training Stability}
\label{sec:stability}

\textbf{Post-warmup training setup.} Starting from OpenFold3 Preview2 weights, we first train with cuEquivariance at crop 640 for 1,000 optimizer updates, following OpenFold3's default linear learning-rate warmup from zero to a configured peak of $1.8\times10^{-3}$. Five backend configurations then resume this shared checkpoint, including its optimizer and scheduler states, on one H100 with BF16 mixed precision and one sample per update. Each run uses seed 42 and is scheduled for 400 further updates at crop 384, 640, or 768; recorded sample order matches across configurations within each crop. TurboPairFormer and cuEquivariance each supply both attention and triangle multiplication (TriMul), while TriFast and MegaFold attention use cuEquivariance TriMul. Unlike Section~\ref{sec:performance}, OpenFold3 Triton attention here uses PyTorch TriMul. These runs use archived backend builds; Appendix~\ref{tpfapp:stability-protocol} specifies their versions, data, and restoration protocol. \textbf{Results.} OpenFold3 Triton first records NaN loss at resumed updates 7, 10, and 8 for crops 384, 640, and 768, respectively (Figure~\ref{fig:stability} in Appendix~\ref{sec:tpf-supp-stability}). MegaFold completes 400 updates at each crop but reaches a raw loss of 53.9 at crop 768. TurboPairFormer, cuEquivariance, and TriFast remain finite throughout the 400-update interval. TurboPairFormer's mean raw training losses are 1.640, 1.603, and 1.603, compared with 1.646, 1.608, and 1.598 for cuEquivariance. Each configuration has one trajectory per crop, so these observations characterize stability after warmup under this shared-checkpoint protocol; they do not establish convergence or final model quality.


\subsection{Fine-Tuning Validation}
\label{sec:tpf-validation}

\begingroup
\setlength{\intextsep}{6pt plus 2pt minus 2pt}
\begin{figure}[!htbp]
\centering
\includegraphics[width=0.94\linewidth,trim=0 6 0 5,clip]{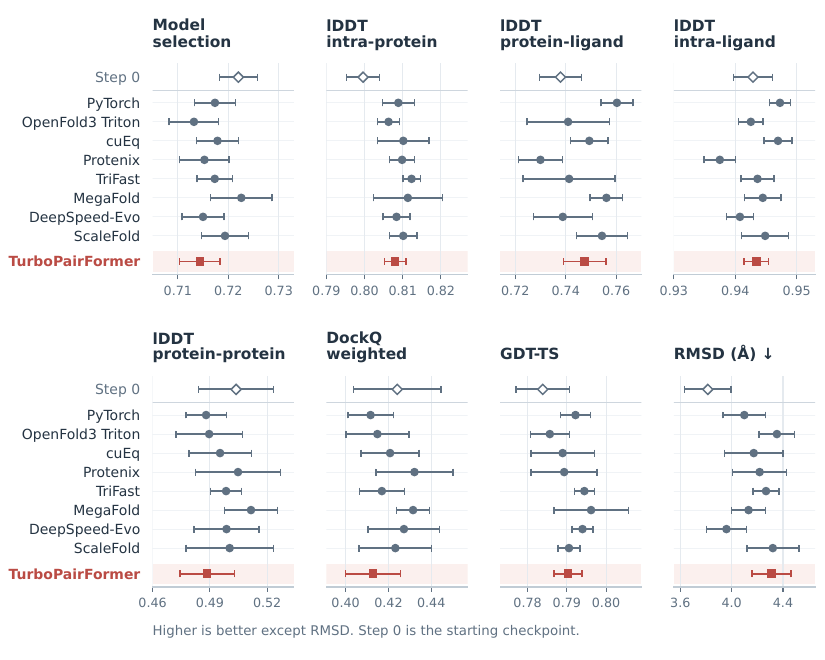}
\caption{\textbf{Preview2 validation after fine-tuning on 16 H100 GPUs.} Points show means and sample SDs over five evaluation seeds on 85 kinase validation complexes after 1,000 global training steps, with step 0 as a reference. Structure metrics average top-1 and best-of-five values; model selection combines lDDT and RASA means (Appendix~\ref{sec:tpf-validation-details}). Higher is better except for RMSD.}
\label{fig:tpf-validation-two-node}
\end{figure}
\endgroup

\textbf{Training setup.} We fine-tune OpenFold3 Preview2~\citep{openfold3preview2} on 9,297 kinase PDB complexes for 1,000 global steps with each of the eight backend configurations of Section~\ref{sec:performance} and with the stock PyTorch configuration, on two nodes with eight H100 GPUs each (settings in Table~\ref{tpfapp:tab:training-settings}). Each final EMA checkpoint and the Preview2 starting point (step 0) are evaluated on the same 85-complex kinase validation cohort using five evaluation seeds (0--4), which fix the diffusion noise of every target. \textbf{Validation results.} The model-selection means of the nine fine-tuned checkpoints span 0.713--0.723, with a pooled standard deviation of 0.0045 over evaluation seeds (Figure~\ref{fig:tpf-validation-two-node}). In paired comparisons over targets, TurboPairFormer's score differs from those of seven other fine-tuned configurations by $-0.005$ to $+0.001$, with 95\% bootstrap intervals that include zero. Its score is 0.0082 below MegaFold's, with an interval of $[-0.0169,-0.0002]$ (Appendix~\ref{sec:tpf-validation-details}). The measured model-selection scores are broadly similar, with a small advantage for MegaFold in this paired comparison; these results do not establish equivalence across backend configurations.

\section{Related Work}
\label{sec:related}

\paragraph{Triangular attention and kernel dataflow.}
AlphaFold2 and AlphaFold3 use triangular updates to pair representations~\citep{jumper2021alphafold2,abramson2024alphafold3}; FastFold, ScaleFold, and MegaFold accelerate these models through kernel and system optimizations~\citep{cheng2022fastfold,zhu2024scalefold,la2025megafold}. FlashAttention and its successors establish tiled recomputation and hardware-specific pipelines~\citep{dao2022flashattention,dao2023flashattention2,shah2024flashattention3,zadouri2026flashattention4}. NeuralPLexer3 avoids explicit bias replication~\citep{qiao2026neuralplexer3}, FlexAttention compiles PyTorch score modifications~\citep{dong2025flexattention,pytorch}, and FlashBias factorizes some structured biases exactly and approximates general biases~\citep{wu2025flashbias}. TurboPairFormer retains the dense shared bias. The inspected DeepSpeed-Evo and ScaleFold kernels already recompute probabilities once, but combine bias gradients through atomic accumulation and full score-gradient storage, respectively~\citep{song2023deepspeed4science,zhu2024scalefold}. Our distinction is combining key-tile parallelism with ordered query and pair-bias partials (Appendix~\ref{tpfapp:comparison}). \paragraph{Deterministic gradient reduction.}
Order-sensitive floating-point summation~\citep{higham1993summation} motivates either order-independent accumulators~\citep{ahrens2020reproducible,collange2015reproducible} or prescribed reduction sequences. FlashAttention-2's deterministic mode uses FP32 query-gradient buffers per split (Appendix~\ref{tpfapp:mechanisms}); FlashAttention-3 and FlashAttention-4 order in-place updates with semaphores~\citep{qiang2026dash,zadouri2026flashattention4}, and DASH improves the scheduling of these ordered updates~\citep{qiang2026dash}. TurboPairFormer materializes partials for both query and shared pair-bias gradients, allowing key-tile blocks to write independently before fixed FP32 reductions. This avoids waiting for accumulation turns, but requires partial storage; the query-partial buffer remains cubic in crop length (Appendices~\ref{tpfapp:memory} and~\ref{tpfapp:determinism}). 
\paragraph{Precision of the backward correction.}
Mixed-precision training allocates precision selectively~\citep{micikevicius2018mixed}, while multi-component representations retain rounding information in BF16 arithmetic and training states~\citep{henry2019leveraging,yu2024collage}. \citet{qiu2026lowprecision} connect attention-output error to the backward correction $\delta=\langle dO,O\rangle$ and adjust the softmax shift to mitigate biased rounding. Attn-QAT computes an auxiliary output from unquantized probabilities for this correction in FP4 attention~\citep{zhang2026attnqat}. TurboPairFormer instead extracts a BF16 rounding residual from the existing FP32 output accumulator and reconstructs the saved output for $\delta$. This preserves the public BF16 output and avoids another $PV$ product; it targets output-conversion error, while other rounding errors remain (Appendix~\ref{tpfapp:residual}).

\section{Conclusion}
\label{sec:conclusion}

TurboPairFormer combines one probability recomputation for all four triangular-attention gradients, fixed-order partial reductions, output-residual compensation, and shared-bias staging on NVIDIA Hopper GPUs. Under the tested conditions, it achieves bitwise gradient repeatability and the lowest mean query, key, and pair-bias gradient RMSE against an FP64 reference among eight attention implementations. Integrated with our triangle multiplication kernels, it achieves the lowest measured GPU computation time per optimizer step among eight OpenFold3 backend configurations on 16 H100 GPUs. Relative to our deterministic implementation, the fastest tested non-deterministic variants take 2.8--4.1\% less backward time at the Pairformer head dimension $D=32$ and 11.0--13.7\% less at $D=16$. The query-partial buffer grows cubically with crop size; reducing this storage while preserving a fixed accumulation order is the main direction for future work.
\section*{Acknowledgments}
\label{sec:acknowledgments}
We thank Ravali Reddy, Thomas Bordes, Paul Zhao, Niranjan Hira, Stephen Balaban, Michael Balaban, and Dave Ward from Lambda, as well as Mallory Tollefson from the Open Molecular Software Foundation, for their support throughout this project.
\bibliography{sections/ref}
\bibliographystyle{preprint}

\appendix
\providecommand{\Require}{\REQUIRE}
\providecommand{\Ensure}{\ENSURE}
\providecommand{\State}{\STATE}
\providecommand{\For}{\FOR}
\providecommand{\EndFor}{\ENDFOR}
\providecommand{\If}{\IF}
\providecommand{\Else}{\ELSE}
\providecommand{\EndIf}{\ENDIF}
\DeclareRobustCommand{\chg}[1]{#1}

\appendix
\raggedbottom
\newtheorem{tpfproposition}{Proposition}[section]
\numberwithin{equation}{section}

\section*{Appendix}

This appendix gives the technical details, proofs, and additional experiments that support the main text. It is organized as follows.
 Appendix~\ref{tpfapp:notation} lists the notation.
 Appendix~\ref{tpfapp:outer} explains how pair features form slices that share one pair bias, with a worked example.
 Appendix~\ref{tpfapp:forward} describes the forward pipeline: online softmax, staging, and the saved output residual.
 Appendix~\ref{tpfapp:backward} derives the backward gradient identities, describes the key-tile partitioning and ordered accumulation, and proves that the result does not depend on the workspace schedule. It also bounds the error of the output-residual compensation.
 Appendix~\ref{tpfapp:memory} analyzes the workspace and its effect on peak training memory.
 Appendix~\ref{tpfapp:comparison} compares our kernels with existing attention kernels and their recomputation and accumulation strategies.
 Appendix~\ref{tpfapp:measurement} gives additional end-to-end results: step latency on eight H200 GPUs, training loss, numerical accuracy, deterministic replay, and sixteen-GPU fine-tuning validation.
 Appendix~\ref{tpfapp:operator} reports single-GPU operator measurements for triangular attention and triangle multiplication, plus a forward dataflow ablation.
 Appendix~\ref{tpfapp:determinism} explains which gradients need a cross-block reduction, compares ways to make such reductions deterministic, and measures their cost.
 Appendix~\ref{tpfapp:trimul} presents the design of the triangle multiplication kernel and a component ablation.

\section{Notation}
\label{tpfapp:notation}

Table~\ref{tpfapp:tab:notation} defines the variables used in the method and appendix. A colon in a tensor subscript selects all entries along that axis, and $\top$ denotes transpose. Parenthesized superscripts identify key tiles or outer groups rather than powers.

\begingroup
\small
\renewcommand{\arraystretch}{1.08}
\setlength{\arrayrulewidth}{0.3pt}
\setlength{\tabcolsep}{6pt}
\setlength{\LTleft}{0pt}
\setlength{\LTright}{0pt}
\setlength{\LTcapwidth}{\linewidth}
\begin{longtable}{|>{\centering\arraybackslash}p{\dimexpr0.30\linewidth-2\tabcolsep-1.5\arrayrulewidth\relax}|>{\raggedright\arraybackslash}p{\dimexpr0.70\linewidth-2\tabcolsep-1.5\arrayrulewidth\relax}|}
\caption{Symbols used in the method and appendix. Tensor shapes are logical; batch and head indices are omitted where stated.}
\label{tpfapp:tab:notation}\\
\hline
\textbf{Notation} & \textbf{Meaning} \\
\hline
\endfirsthead
\hline
\textbf{Notation} & \textbf{Meaning} \\
\hline
\endhead
\endfoot
\endlastfoot
$B$, $N$, $N_o$ & Examples per GPU, tokens per crop, and number of fixed-token attention slices. The label $o$ in $N_o$ denotes the outer axis; $N_o=N$. \\ \hline
$H$, $D$, $C_z$ & Heads, features per head, and features per pair; the label $z$ refers to the pair tensor $Z$. \\ \hline
$o,q,k$; $i,j$ & Token indices: the fixed token (outer index), and the query and key positions within its slice. The triangle example uses $i=o$ and $j=q$. A subscript $o$ alone selects one fixed-token slice. \\ \hline
$d$; $t,\ell,c$ & Feature coordinate; key tile, outer group, and forward 16-key chunk indices. There are four chunks per forward tile. In the memory analysis, $t$ denotes an execution point. \\ \hline
$Z$ & Pair features of shape $[B,N,N,C_z]$. $Z_{ij:}$ selects one pair, and $Z_{o,:}$ (also written $Z_{o,:,:}$) the slice of a fixed token $o$; numeric subscripts such as $Z_{2,4,:}$ select individual pairs. \\ \hline
$Q,K,V,O$ & Query, key, value, and attention output tensors, each of shape $[B,N_o,H,N,D]$; for one example, slice, and head, each is an $N\times D$ matrix. \\ \hline
$\mathrm{bias}$, $M$ & Shared pair bias of shape $[B,H,N,N]$ and additive mask. $M_{oqk}$ selects a score mask; the key-only form is $M_{ik}$. \\ \hline
$T,S,P$ & Unscaled query--key scores, scores after scaling and adding bias and mask, and softmax probabilities. Their tensor indices follow $o,q,k$. \\ \hline
$X,dX$ & Any tensor and its loss gradient, including $dZ$. The prefix $d$ means gradient; a subscript $d$ means feature coordinate. Local $dK$ accumulators are unscaled until the final store. \\ \hline
$\mathrm{LSE}$, $\delta$ & Saved log-sum-exp and softmax backward correction for each row. Subscript $oq$ selects query row $q$ in fixed-token slice $o$. \\ \hline
$\tau_q,\tau_k$ & Query and key tile sizes. Superscript $\mathrm f$ denotes forward: $\tau_q^{\mathrm f}=\tau_k^{\mathrm f}=64$. Backward uses $\tau_q=32$, $\tau_k=64$. Here $q,k$ label the tile axes. \\ \hline
$\mathcal K_t$, $J_K$, $J_Q$ & Keys in backward tile $t$, and numbers of backward key and query tiles; the labels $K$ and $Q$ refer to keys and queries. $\mathcal K_t^{\mathrm f}$ is the corresponding forward tile. \\ \hline
$\mathcal L_\ell$, $G$, $L$ & Fixed tokens in outer group $\ell$, tokens per group ($G=4$), and number of groups. Groups need not follow chain boundaries. \\ \hline
$o_0,o_1,o_2,o_3$ & Four fixed-token positions within one outer group. $dS_{o_0},\ldots,dS_{o_3}$ are their contributions to the same bias element; invalid positions contribute zero. \\ \hline
$B_c$, $Q_s$, $R$ & Examples per batch chunk, query rows per query chunk, and outer groups per launch. Here $c$ and $s$ label the two chunk sizes, not the chunk index $c$ used elsewhere. \\ \hline
$dQ_{\mathrm{part}}^{(t)}$ & Unscaled query-gradient partial from key tile $t$; $\mathrm{part}$ means partial. Added subscripts $oq:$ select its output row. The identities use exact partials; algorithms use their stored BF16 values. \\ \hline
$d\,\mathrm{bias}_{\mathrm{part}}^{(\ell)}$ & Bias-gradient partial from outer group $\ell$. Added subscripts $qk$ select one output element. The identities use exact partials; algorithms use their stored BF16 values. \\ \hline
$\widehat{d\,\mathrm{bias}}_{\mathrm{part}}^{(\ell)}$, $d\,\mathrm{bias}_{\mathrm{acc}}^{(\ell)}$ & The hat marks a stored BF16 bias partial. The label $\mathrm{acc}$ marks the FP32 sum through group $\ell$; state $0$ already contains group $0$'s partial. \\ \hline
$\mu$, $\lambda$ & Running row maximum and softmax normalizer. Subscripts $oq$ identify the row; a prime marks the value after the current key tile. \\ \hline
$\rho$, $\mathrm{acc}$ & Row rescaling factor and unnormalized FP32 output sum. Subscripts $oq$ or $oq:$ select a row scalar or vector; a prime marks the updated value. \\ \hline
$O_{32}$, $O_{16}$, $E_{16}$ & FP32 output before conversion, public BF16 output, and saved BF16 approximation of the output-rounding residual. Subscripts $32$ and $16$ label precision, not tensor coordinates; $16$ denotes BF16. \\ \hline
$\widetilde O_{32}$, $O^*$ & Output reconstructed from $O_{16}$ and $E_{16}$, and exact attention output. The tilde marks reconstruction; the star marks the exact result. \\ \hline
$x,y,e$; $\widehat e,\widetilde x$ & One FP32 output entry, its BF16 rounding, and the discarded residual. The hat marks the rounded BF16 residual; the tilde marks the reconstructed entry. \\ \hline
$u_{16}$, $u_{32}$ & BF16 and FP32 unit roundoffs under the stated normal-range assumptions; the subscripts label precision. \\ \hline
$\operatorname{RN}_{16}$, $\operatorname{RN}_{32}$, $\operatorname{FP32}(\cdot)$, $\sum^{\mathrm{FP32}}$ & Round-to-nearest conversion to BF16 and to FP32, exact conversion of a BF16 value to FP32, and a sum evaluated in FP32. \\ \hline
$\circ$, $\langle\cdot,\cdot\rangle$, $\operatorname{softmax}$, $\operatorname{diag}$, $\operatorname{rowsum}$, $\operatorname{rowmax}$, $X^\top$ & Elementwise product, inner product over the feature coordinate, softmax over the keys of each query row, the diagonal matrix of a vector, the sum and the maximum over each row, and the transpose. $S_c$, $V_{o,c}$, and $P_c$ restrict a tile to the keys of forward chunk $c$. \\ \hline
$\mathbb R$ & Real numbers; for example, one score tile lies in $\mathbb R^{\tau_q\times\tau_k}$. \\ \hline
$\Theta(\cdot)$, $\|\cdot\|_1$, $\|\cdot\|_2$ & Asymptotic order, and the sum of absolute values and the Euclidean norm over all entries. The relative $L_2$ error of a tensor $X$ against a reference $X'$ is $\|X-X'\|_2/\|X'\|_2$. \\ \hline
$\widehat\delta$, $\eta$ & Approximate row correction (marked by the hat) and its difference from the reference correction. $\eta_{oq}$ is one row's difference; $\eta_o$ collects the differences for slice $o$. \\ \hline
$\Delta dX$ & Change in a gradient when only the row correction changes; the prefix $\Delta$ marks the difference. \\ \hline
$W_Q$ & Bytes for all BF16 query-gradient partials; $Q$ labels the query gradient. \\ \hline
$W_{\mathrm{bias,all}}$, $W_{\mathrm{bias,res}}$ & Bytes for all BF16 bias partials or only the resident subset. The labels $\mathrm{all}$ and $\mathrm{res}$ mean all and resident. \\ \hline
$W_E$, $W_{dS}$ & Bytes for one attention call's BF16 output residual and for a complete FP32 score gradient; $E$ and $dS$ label these tensors. \\ \hline
$A(t)$, $W(t)$; $P_0$, $P_1$ & Bytes of other live tensors and of temporary operator workspace at execution point $t$; peak memory without and with this workspace. Subscripts $0,1$ label the two cases; these are not attention probabilities. \\ \hline
\end{longtable}
\endgroup

\section{Biomolecular Inputs and Tensor Organization}
\label{tpfapp:outer}

\subsection{Tokens, Pairs, and the Outer Index}

AlphaFold3 assigns one token to each standard polymer residue and one to each heavy atom of other components~\citep{abramson2024alphafold3}, and a crop selects $N$ tokens, which can include noncontiguous residues from several chains. The pair representation of one example is $Z\in\mathbb R^{N\times N\times C_z}$, where $Z_{ij:}$ is a learned feature vector for the ordered token pair $(i,j)$. These features encode pairwise information rather than coordinates or chemical bonds, and masks decide which pairs can attend. Starting-node triangular attention fixes the first pair index, the \emph{outer index} $o$: the slice $Z_{o,:}$ holds the $N$ pairs between token $o$ and every token, and attention runs along its second index, so a crop has $N$ slices, each with an $N\times N$ score matrix per head. Ending-node attention fixes the second index and uses the same core after permuting the pair axes~\citep{openfold3preview}. Section~\ref{tpfapp:worked-example} follows one score through a seven-token example.

\subsection{Shared Pair Bias and Its Gradient}

For starting-node attention, suppress batch and head indices and let the projections act on normalized pair features. A score for updating pair $(i,j)$ using pair $(i,k)$ has the form
\begin{equation}
S_{ijk}=\frac{1}{\sqrt{D}}\langle Q_{ij:},K_{ik:}\rangle+\mathrm{bias}_{jk}+M_{ik}.
\label{tpfapp:eq:triangle}
\end{equation}
The score connects the pair relations $(i,j)$, $(i,k)$, and $(j,k)$, which form the triangle. Here $i$ is the outer index and $j,k$ are the query and key indices. The same $\mathrm{bias}_{jk}$ is used for every $i$, so by the chain rule
\begin{equation}
d\,\mathrm{bias}_{jk}=\sum_i dS_{ijk}.
\label{tpfapp:eq:bias-broadcast}
\end{equation}
The forward pass evaluates outer slices independently, whereas the shared bias gradient couples them through this sum.

\subsection{A Worked Example}
\label{tpfapp:worked-example}

Figure~\ref{tpfapp:fig:worked-example} follows one attention score through a crop of seven tokens. Steps 1--4 trace the forward pass of query row $q=4$ in slice $o=2$, and step 5 shows the three backward sums that the gradient of the same score joins.

\begin{figure}[!t]
\centering
\includegraphics[width=\linewidth]{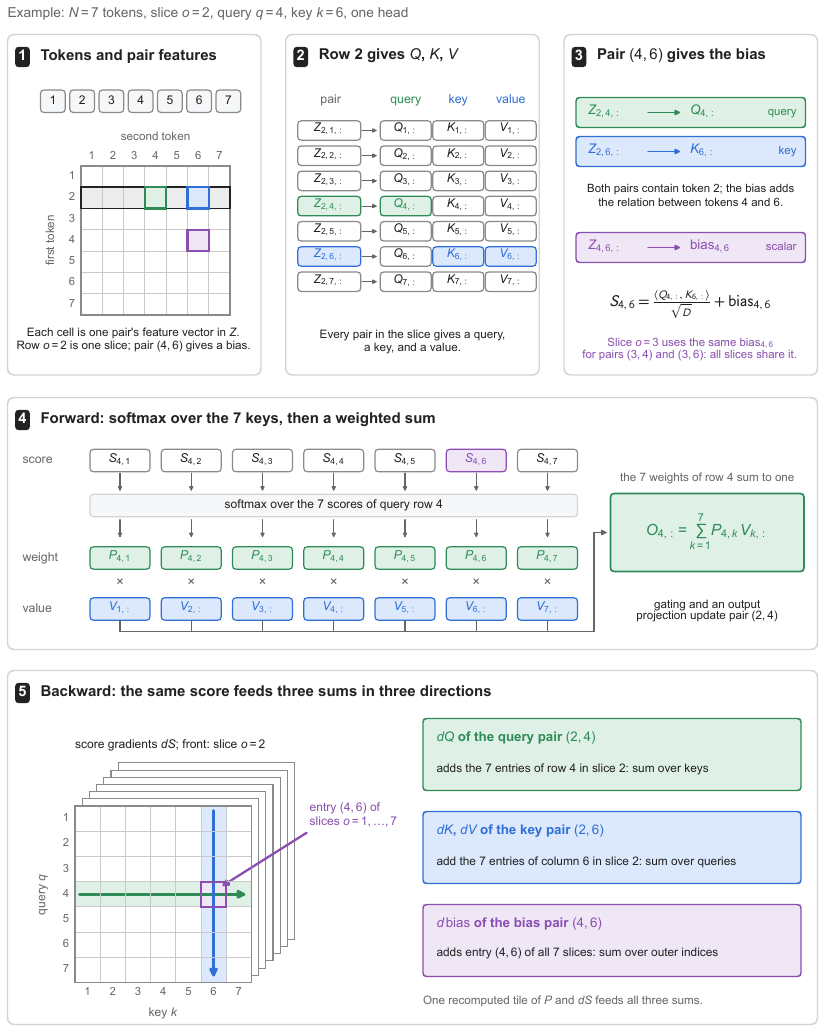}
\caption{\textbf{Worked example of triangular attention.} We follow one attention score through a crop of seven tokens with one head; numbers in subscripts are token positions. (1)~$Z$ stores one feature vector for every pair of tokens, and fixing token $o=2$ takes row 2 of $Z$ as one slice. (2)~Each pair in this row is projected into a query, a key, and a value; the pair $(2,4)$ gives the query $Q_{4,:}$, and the pair $(2,6)$ gives the key $K_{6,:}$ and the value $V_{6,:}$. (3)~Their score also adds a bias taken from the third pair, $(4,6)$, and every slice uses this same bias for its positions 4 and 6. (4)~The seven scores of query 4 go through a softmax, and the resulting weights average the seven values into the output that updates the pair $(2,4)$. (5)~In the backward pass, this one score enters three different sums: over keys for $dQ$, over queries for $dK$ and $dV$, and over slices for $d\,\mathrm{bias}$. Other queries and slices work the same way; masks and the normalization before the projections are omitted.}
\label{tpfapp:fig:worked-example}
\end{figure}

\section{Forward Pipeline}
\label{tpfapp:forward}

Algorithm~\ref{tpfapp:alg:forward} follows one query tile through the forward kernel.

\subsection{Online Softmax and Chunked Value Aggregation}
\label{tpfapp:forward-softmax}

The forward kernel evaluates Equation~(\ref{tpfapp:eq:triangle}) with tiled online softmax~\citep{milakov2018onlinesoftmax}, which avoids storing the full score and probability matrices~\citep{rabe2021selfattention,dao2022flashattention}. A thread block of four warps owns $\tau_q^{\mathrm f}=64$ query rows of one slice and head, and traverses key tiles of $\tau_k^{\mathrm f}=64$ keys in ascending order. After key tiles $0,\dots,t-1$, it carries for each query row a running maximum $\mu_{oq}$, a running normalizer $\lambda_{oq}$, and an unnormalized FP32 output $\mathrm{acc}_{oq:}$. With $\mathcal K^{\mathrm f}_t$ the keys of tile $t$, the update is
\begin{align}
\mu'_{oq}&=\max\Bigl(\mu_{oq},\max_{k\in\mathcal K^{\mathrm f}_t}S_{oqk}\Bigr),
&\rho_{oq}&=\exp\bigl(\mu_{oq}-\mu'_{oq}\bigr),\\
\lambda'_{oq}&=\rho_{oq}\lambda_{oq}
  +\sum_{k\in\mathcal K^{\mathrm f}_t}\exp\bigl(S_{oqk}-\mu'_{oq}\bigr),\\
\mathrm{acc}'_{oq:}&=\rho_{oq}\,\mathrm{acc}_{oq:}
  +\sum_{k\in\mathcal K^{\mathrm f}_t}
   \exp\bigl(S_{oqk}-\mu'_{oq}\bigr)V_{ok:},
\label{tpfapp:eq:online}
\end{align}
and after the last key tile the block emits
\begin{equation}
O_{32,oq:}=\frac{\mathrm{acc}_{oq:}}{\lambda_{oq}},\qquad
\mathrm{LSE}_{oq}=\mu_{oq}+\log\lambda_{oq}.
\label{tpfapp:eq:forward-out}
\end{equation}
A masked score of $-\infty$ receives zero probability, and a row without valid keys gets $O_{32}=0$ and $\mathrm{LSE}=-\infty$; the saved $\mathrm{LSE}$ serves the backward recomputation. Each score tile is consumed in ascending 16-key chunks indexed by $c$: the exponentials of a chunk are computed just before its values are aggregated, and each thread reads the bias entries of its score fragment from shared memory in the register layout of the score accumulator. Ascending chunk and key-tile traversal fixes the accumulation order in Equation~(\ref{tpfapp:eq:online}).

\begin{algorithm}[!t]
\caption{\textsc{TurboPairFormer} forward pass, thread-block view}
\label{tpfapp:alg:forward}
\small
\algrenewcommand\algorithmicindent{1.0em}
\begin{algorithmic}[1]
\Require $Q,K,V\in\mathbb R^{B\times N_o\times H\times N\times D}$, $\mathrm{bias}\in\mathbb R^{B\times H\times N\times N}$, $M$ in HBM; one query tile of $\tau_q^{\mathrm f}$ rows in slice $o$; key tiles $\mathcal K^{\mathrm f}_t$ of $\tau_k^{\mathrm f}$ keys.
\Statex \textbf{Stage the query tile once}
\State Load $Q_o$ from HBM to shared memory, then to registers.
\State Prefetch $K_o$, $V_o$, $M_o$, and the bias tile of key tile $0$ into shared-memory buffer $0$.
\State On chip, set $\mu=-\infty$, $\lambda=0$, and $\mathrm{acc}=0\in\mathbb R^{\tau_q^{\mathrm f}\times D}$ in FP32.
\Statex \textbf{Visit the key tiles}
\For{each key tile $t$, in ascending order}
  \State Wait for $K_o$ of tile $t$; on chip, compute $T=Q_oK_o^\top$ (WGMMA).
  \State Prefetch $K_o$, $V_o$, $M_o$, and the bias tile of key tile $t+1$ into the other buffer.
  \State Wait for the bias tile of tile $t$; on chip, compute $S=T/\sqrt D+\mathrm{bias}+M_o$.
  \State On chip, compute $\mu'=\max(\mu,\operatorname{rowmax}(S))$.
  \State Per row, set $\rho=0$ if $\mu=-\infty$, otherwise $\rho=\exp(\mu-\mu')$.
  \State On chip, set $\lambda\gets\rho\lambda$, $\mathrm{acc}\gets\rho\,\mathrm{acc}$, and $\mu\gets\mu'$.
  \State Wait for $V_o$ of tile $t$.
  \For{each 16-key chunk $c$, in ascending order}
    \State Set each $P_c$ entry to zero if its score is $-\infty$, otherwise to $\exp(S_c-\mu)$.
    \State On chip, compute $\mathrm{acc}\gets\mathrm{acc}+P_cV_{o,c}$ (mma.sync).
  \EndFor
  \State On chip, compute $\lambda\gets\lambda+\operatorname{rowsum}(P)$.
\EndFor
\Statex \textbf{Write the outputs}
\State Per row with $\lambda>0$, compute $O_{32}=\mathrm{acc}/\lambda$ and $\mathrm{LSE}=\mu+\log\lambda$.
\State For rows with $\lambda=0$, set $O_{32}=0$ and $\mathrm{LSE}=-\infty$.
\State On chip, round $O_{32}$ to $O_{16}$ in BF16, and round the residual $O_{32}-O_{16}$ to $E_{16}$ in BF16.
\State Write $O_{16}$, $E_{16}$, and $\mathrm{LSE}$ to HBM.
\end{algorithmic}
\end{algorithm}

\subsection{Staging and Output Residual}
\label{tpfapp:forward-lanes}
\label{tpfapp:forward-residual}

The Tensor Memory Accelerator (TMA) loads the key, value, and pair-bias tiles into two alternating shared-memory buffers, so the transfers for key tile $t+1$ overlap the computation on tile $t$; the key mask is prefetched with ordinary loads. Scores use WGMMA and value aggregation uses \texttt{mma.sync}, both with BF16 operands and FP32 accumulation, and the kernel is instantiated for $D\in\{16,32,64\}$. The epilogue rounds the FP32 output to the public BF16 output $O_{16}$ and also saves its BF16 rounding residual $E_{16}$ (Equation~(\ref{tpfapp:eq:residual-rounding})), which Section~\ref{tpfapp:residual} analyzes.

\section{Backward Identities and Ordered Accumulation}
\label{tpfapp:backward}

\subsection{Gradient Identities}
\label{tpfapp:gradients}

Suppressing batch and head indices, the attention core is
\begin{align}
S_{oqk}&=\frac{1}{\sqrt{D}}\langle Q_{oq:},K_{ok:}\rangle+\mathrm{bias}_{qk}+M_{oqk},\\
P_{oq:}&=\operatorname{softmax}(S_{oq:}),
&O_{oq:}&=\sum_k P_{oqk}V_{ok:}.
\end{align}
For a scalar loss and incoming derivative $dO$, the backward identities are, in exact arithmetic,
\begin{align}
dP_{oqk}&=\langle dO_{oq:},V_{ok:}\rangle,\\
\delta_{oq}&=\sum_k P_{oqk}dP_{oqk}
             =\langle dO_{oq:},O_{oq:}\rangle,\label{tpfapp:eq:delta}\\
dS_{oqk}&=P_{oqk}(dP_{oqk}-\delta_{oq}),\label{tpfapp:eq:ds}\\
dQ_{oq:}&=\frac{1}{\sqrt{D}}\sum_k dS_{oqk}K_{ok:},
&dK_{ok:}&=\frac{1}{\sqrt{D}}\sum_q dS_{oqk}Q_{oq:},\label{tpfapp:eq:qk}\\
dV_{ok:}&=\sum_q P_{oqk}dO_{oq:},
&d\,\mathrm{bias}_{qk}&=\sum_o dS_{oqk}.\label{tpfapp:eq:vb}
\end{align}

These identities apply to rows with at least one valid key. For fully masked rows, adopting the zero-output convention gives zero attention gradients.

\paragraph{Derivation.} Differentiating $O=PV$ gives $dP=dO\,V^\top$ and $dV=P^\top dO$. The softmax Jacobian in each row is $\operatorname{diag}(P)-PP^\top$, yielding Equation~(\ref{tpfapp:eq:ds}). Differentiating the bilinear score gives Equation~(\ref{tpfapp:eq:qk}); differentiating the broadcast bias gives the sum over $o$ in Equation~(\ref{tpfapp:eq:vb}). Finally, exchanging the finite sums over $k$ and feature index $d$ gives
\[
\sum_k P_{oqk}\sum_d dO_{oqd}V_{okd}
=\sum_d dO_{oqd}\sum_k P_{oqk}V_{okd}
=\langle dO_{oq:},O_{oq:}\rangle.
\]
This last identity makes the row correction available before a key-tile traversal when the forward output is retained.

The three reduction axes determine the backward ownership problem: $dQ$ sums over keys, $dK,dV$ sum over queries, and $d\,\mathrm{bias}$ sums over outer indices. The next construction shares the recomputed intermediates while assigning each partial a unique writer.

\subsection{Key-Tile Ownership and Gradient Partitioning}

Partition keys into disjoint tiles $\mathcal K_t$ and outer indices into disjoint outer groups $\mathcal L_\ell$. A thread block owns one key tile, one head, and one outer group, and then traverses query tiles. Separate $dK,dV$ accumulators retain the contribution for each owned outer index. The remaining reductions are partitioned into the exact partials
\begin{align}
dQ_{\mathrm{part},oq:}^{(t)}&=\sum_{k\in\mathcal K_t}dS_{oqk}K_{ok:},\\
d\,\mathrm{bias}_{\mathrm{part},qk}^{(\ell)}&=\sum_{o\in\mathcal L_\ell}dS_{oqk}.
\end{align}

\begin{tpfproposition}[Partitioned backward]
\label{tpfapp:prop:partition}
If every logical $(o,q,k)$ contribution is computed exactly once, then
\begin{equation}
dQ_{oq:}=\frac{1}{\sqrt{D}}\sum_t dQ_{\mathrm{part},oq:}^{(t)},\qquad
d\,\mathrm{bias}_{qk}=\sum_\ell d\,\mathrm{bias}_{\mathrm{part},qk}^{(\ell)},
\label{tpfapp:eq:partials}
\end{equation}
and the thread blocks that own the key tiles obtain $dK,dV$ by completing their query traversals. The computation can reuse one recomputation of each probability tile for all four gradients without retaining the complete score gradient.
\end{tpfproposition}

\begin{proof}
Substitute the disjoint key and outer partitions into the finite sums in Equations~(\ref{tpfapp:eq:qk})--(\ref{tpfapp:eq:vb}). Every term belongs to exactly one partial. Once $\delta$ is available, a recomputed tile $P$ supplies $dV$; its associated tile $dS$ supplies $dQ$, $dK$, and $d\,\mathrm{bias}$. A thread block that owns a key tile visits all query tiles required by its $dK,dV$ entries. The remaining sums are exactly Equation~(\ref{tpfapp:eq:partials}).
\end{proof}

Each thread block writes its partials at unique addresses, and separate reduction kernels add them in a prescribed order, so one probability recomputation per tile suffices without floating-point atomics. The partials are stored in BF16 and reduced in FP32. Numerical accuracy is evaluated against a reference, whereas repeatability follows from the fixed operation order below.

\subsection{Backward Kernel Implementation}
\label{tpfapp:algorithms}

Algorithm~\ref{tpfapp:alg:leaf} expands the main kernel of Algorithm~\ref{alg:tpf_backward} into thread-block steps, and Figure~\ref{tpfapp:fig:backward-dataflow} shows where each tensor resides. A thread block of two warpgroups keeps the key-tile operands in shared memory, double-buffers the query-tile operands, and forms $S$, $P$, $dP$, and $dS$ once per tile for all four gradients. Each warpgroup adds the score gradients of its two slices into an FP32 bias subtotal, and the two subtotals form the outer group's partial. Probabilities are zero in fully masked rows and at invalid outer indices, so they add nothing. The row correction $\widehat\delta$ comes from a preprocessing kernel, and enabling residual compensation changes only this scalar. The reduction kernels give each output element one owner and add its partials in ascending key-tile order for $dQ$ and in ascending outer-group order for $d\,\mathrm{bias}$ (Equations~(\ref{tpfapp:eq:partials}) and~(\ref{tpfapp:eq:ordered-fold})).

\begin{figure}[!t]
\centering
\includegraphics[width=\linewidth]{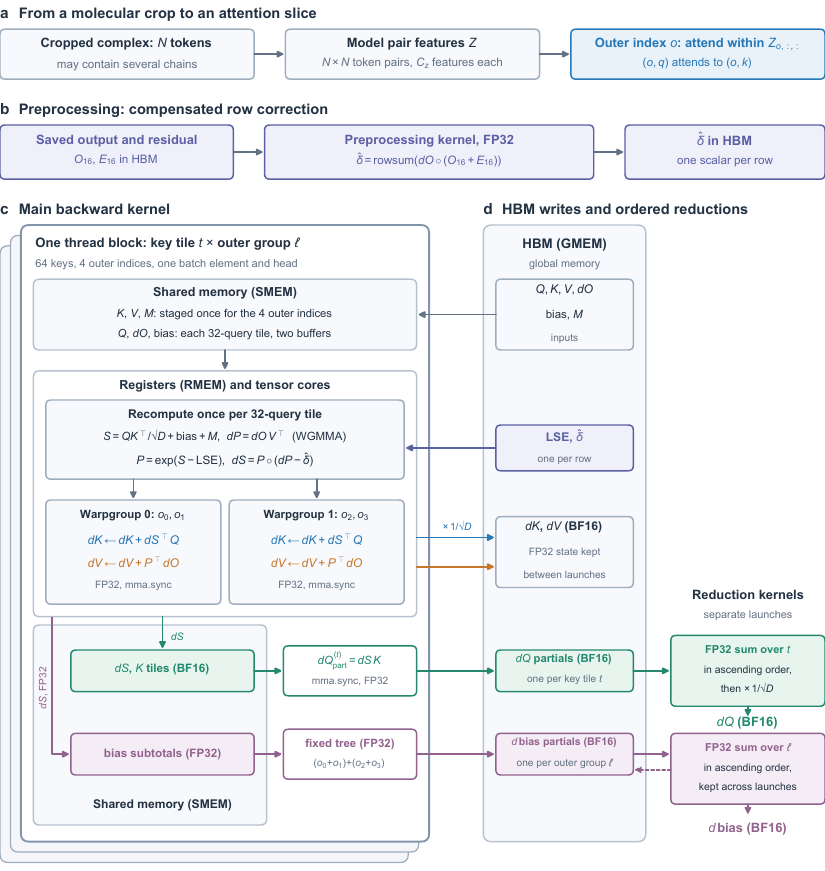}
\caption{\textbf{Backward data movement.} (a) Fixing the outer index $o$ selects the slice $Z_{o,:,:}$, in which the query pair $(o,q)$ attends to the pairs $(o,k)$. (b) A preprocessing kernel computes the compensated row correction $\hat\delta$ in FP32 and writes it to HBM. (c) Each thread block of the main backward kernel owns one key tile and one outer group, and further thread blocks run concurrently on other SMs. It stages $K$, $V$, and $M$ once and double-buffers $Q$, $dO$, and the bias tile in shared memory (SMEM), recomputes $S$, $P$, $dP$, and $dS$ once per query tile, and accumulates $dK$ and $dV$ in FP32 registers (RMEM). The score gradient passes through SMEM to form one query-gradient partial per key tile and, through a fixed FP32 tree, one pair-bias partial per outer group. (d) The thread block writes $dK$, $dV$, and its partials to HBM; separate reduction kernels add the partials in ascending index order in FP32, and the pair-bias partial buffers are reused once consumed (dashed arrow).}
\label{tpfapp:fig:backward-dataflow}
\end{figure}

\begin{algorithm}[!t]
\caption{\textsc{TurboPairFormer} main backward kernel, thread-block view}
\label{tpfapp:alg:leaf}
\small
\algrenewcommand\algorithmicindent{1.0em}
\begin{algorithmic}[1]
\Require $Q,K,V,dO,\mathrm{bias},M,\mathrm{LSE},\widehat\delta$ in HBM; key tile $t$; outer group $\ell$ with outer indices $o_0,\dots,o_3$; the query rows of this launch.
\Statex \textbf{Stage the key tile once}
\State Load $K_o,V_o,M_o$ of key tile $t$ for $o\in\mathcal L_\ell$ from HBM to shared memory.
\State On chip, set $dK_o=dV_o=0$ in FP32, or reload the FP32 values saved by an earlier launch.
\State Prefetch $Q_o,dO_o$ and the bias tile of the first query tile into shared-memory buffer $0$.
\Statex \textbf{Visit the query tiles}
\For{each query tile, in ascending order}
  \State Wait for this tile's buffer; prefetch the next query tile into the other buffer.
  \For{each warpgroup in parallel, and each of its two outer indices $o$ in order}
    \State On chip, compute $S=Q_oK_o^\top/\sqrt D+\mathrm{bias}+M_o$ and $dP=dO_oV_o^\top$ (WGMMA).
    \State On chip, compute $P=\exp(S-\mathrm{LSE})$ and $dS_o=P\circ(dP-\widehat\delta)$.
    \State On chip, compute $dV_o\gets dV_o+P^\top dO_o$ and $dK_o\gets dK_o+dS_o^\top Q_o$ (mma.sync).
    \State Add $dS_o$ to the FP32 bias subtotal of this warpgroup; copy $dS_o$ to shared memory in BF16.
  \EndFor
  \State Synchronize the thread block.
  \State Write $d\,\mathrm{bias}_{\mathrm{part}}^{(\ell)}=(dS_{o_0}+dS_{o_1})+(dS_{o_2}+dS_{o_3})$, the sum of the two subtotals, to HBM in BF16.
  \State On chip, compute $dQ_{\mathrm{part},o}^{(t)}=dS_oK_o$ for each $o$ (mma.sync); write them to HBM in BF16.
\EndFor
\Statex \textbf{Store the key-tile gradients}
\If{this launch covers the last query rows}
  \State Write $dK_o/\sqrt D$ and $dV_o$ to HBM in BF16.
\Else
  \State Write $dK_o$ and $dV_o$ to HBM in FP32, to be reloaded by the next launch.
\EndIf
\end{algorithmic}
\end{algorithm}

\subsection{Workspace Schedule Invariance}
\label{tpfapp:invariance}

To bound the resident workspace, the planner splits the backward pass into launches over consecutive chunks of $B_c$ examples, $Q_s$ query rows, and $R$ outer groups; each query row still attends to all valid keys. These units describe GPU work and memory residency, not protein chains or additional crops.

Let $\widehat{d\,\mathrm{bias}}_{\mathrm{part}}^{(\ell)}$ denote the stored BF16 bias partial for one output element of outer group $\ell$. At a common query and key coordinate, its local reduction is
\begin{equation}
\widehat{d\,\mathrm{bias}}_{\mathrm{part}}^{(\ell)}=\operatorname{RN}_{16}\!\left(
 \operatorname{RN}_{32}\!\left[
  \operatorname{RN}_{32}(dS_{o_0}+dS_{o_1})
  +\operatorname{RN}_{32}(dS_{o_2}+dS_{o_3})\right]\right).
\label{tpfapp:eq:leaf-rounding}
\end{equation}
The two inner additions are the warpgroup subtotals; the remaining FP32 addition combines them before BF16 storage. Invalid outer positions contribute zero. Here $\operatorname{RN}_{16}$ and $\operatorname{RN}_{32}$ denote round-to-nearest conversion to BF16 and FP32. Fix the input bits, binary, device, arithmetic tiles, and instruction semantics, including this grouping, precision, and parenthesization. Exact BF16-to-FP32 conversion gives the ordered sum
\begin{align}
d\,\mathrm{bias}_{\mathrm{acc}}^{(0)}&=\operatorname{FP32}(\widehat{d\,\mathrm{bias}}_{\mathrm{part}}^{(0)}),\\
d\,\mathrm{bias}_{\mathrm{acc}}^{(\ell)}&=\operatorname{RN}_{32}
 \bigl(d\,\mathrm{bias}_{\mathrm{acc}}^{(\ell-1)}+\operatorname{FP32}(\widehat{d\,\mathrm{bias}}_{\mathrm{part}}^{(\ell)})\bigr),
 \quad 1\leq\ell<L.\label{tpfapp:eq:ordered-fold}
\end{align}
The final cast occurs after the last outer group. Across launches, the reduction kernel retains the current FP32 state and resumes at the next outer group. Query partials follow the fixed key-tile order.

\begin{tpfproposition}[Workspace schedule invariance]
\label{tpfapp:prop:invariance}
Suppose a physical schedule changes only residency and execution timing, and satisfies all of the following conditions:
\begin{enumerate}
\item Every logical partial is produced once by the same deterministic floating-point operations and written by one thread block to its own workspace location.
\item Each reduction kernel consumes the same ordered partials. Across a launch boundary it stores and reloads the current FP32 running state without conversion or reassociation.
\item If query traversal spans several launches, $dK,dV$ retain their FP32 accumulator states and continue the same ordered query-tile operations.
\item Stream ordering between the main kernel and the reduction kernels prevents stale reads and buffer reuse before consumption, and the final output casts remain unchanged.
\end{enumerate}
Then all four output gradients are bitwise invariant across such schedules.
\end{tpfproposition}

\begin{proof}
Each schedule evaluates the same directed acyclic graph of deterministic floating-point operations. Induct over its dependencies: input nodes have the same bits; a node with identical operand bits and the same operation has identical result bits. Unique writers and synchronization preserve those operands. In particular, dividing Equation~(\ref{tpfapp:eq:ordered-fold}) into launches only inserts bit-preserving FP32 stores and loads. Query chunks partition $dQ$ and $d\,\mathrm{bias}$ output coordinates, while their contributions to the same $dK,dV$ entries continue the original FP32 recurrences. Batch chunks act on independent batch coordinates. Thus every output node has the same bits.
\end{proof}

The distinction between rescheduling and regrouping is essential. Under FP32 round-to-nearest with ties to even,
\begin{align*}
\operatorname{RN}_{32}\bigl(\operatorname{RN}_{32}(2^{24}+1)+1\bigr)
 &=2^{24},\\
\operatorname{RN}_{32}\bigl(2^{24}+\operatorname{RN}_{32}(1+1)\bigr)
 &=2^{24}+2.
\end{align*}
A launch boundary therefore preserves the current FP32 state instead of introducing a new partial sum: regrouping changes the arithmetic even when both schedules are deterministic.

\subsection{Output Residual Compensation}
\label{tpfapp:residual}

The row correction enters every score gradient in a query row. Retaining the forward output's rounding residual reduces the output-reconstruction error bound and can improve this shared quantity before probability recomputation. The forward epilogue computes
\begin{equation}
O_{16}=\operatorname{RN}_{16}(O_{32}),\qquad
E_{16}=\operatorname{RN}_{16}\!\left(\operatorname{RN}_{32}(O_{32}-O_{16})\right).
\label{tpfapp:eq:residual-rounding}
\end{equation}
The subtraction widens $O_{16}$ to FP32. Once $dO$ is available, a separate row reduction computes
\begin{equation}
\widehat\delta_{oq}
=\sum_d^{\mathrm{FP32}}dO_{oqd}\,
\operatorname{RN}_{32}\!\left(O_{16,oqd}+E_{16,oqd}\right).
\label{tpfapp:eq:delta-evaluation}
\end{equation}
All BF16 operands are widened to FP32; both the reconstruction and the dot product execute in FP32. This scalar is available before the key-tile traversal and requires no probability recomputation. The following proposition isolates its effect on the four gradients.

\begin{tpfproposition}[Propagation of the row correction]
\label{tpfapp:prop:correction}
For rows with at least one valid key, hold $P,dP,Q,K,V,dO$ fixed in exact arithmetic and replace $\delta$ by $\widehat\delta=\delta+\eta$. Then
\begin{align}
\Delta dS_{oqk}&=-P_{oqk}\eta_{oq},
&\sum_k\Delta dS_{oqk}&=-\eta_{oq},\\
\Delta dQ_{oq:}&=-\frac{1}{\sqrt{D}}\eta_{oq}\sum_kP_{oqk}K_{ok:},\\
\Delta dK_{ok:}&=-\frac{1}{\sqrt{D}}\sum_qP_{oqk}\eta_{oq}Q_{oq:},\\
\Delta d\,\mathrm{bias}_{qk}&=-\sum_oP_{oqk}\eta_{oq},
&\Delta dV&=0.
\label{tpfapp:eq:row-error}
\end{align}
\end{tpfproposition}

\begin{proof}
Subtract the two instances of Equation~(\ref{tpfapp:eq:ds}) and substitute the difference into Equations~(\ref{tpfapp:eq:qk})--(\ref{tpfapp:eq:vb}). The row-sum identity follows from $\sum_kP_{oqk}=1$.
\end{proof}

The same $\eta_{oq}$ acts on every key in the row, with $\|\Delta dS_{oq:}\|_1=|\eta_{oq}|$. Before weighting by $P$, the broadcast difference $-\eta_o\mathbf 1^\top$ has rank at most one; the resulting score-gradient difference is $-\operatorname{diag}(\eta_o)P_o$. Its propagation into $dQ,dK$, and $d\,\mathrm{bias}$ depends on the subsequent contractions and possible cancellation. This identifies the shared precision channel that the residual addresses.

To isolate output conversion, let $x$ be one entry of the normalized FP32 forward result $O_{32}$, and define
\[
y=\operatorname{RN}_{16}(x),\quad e=x-y,\quad
\widehat e=\operatorname{RN}_{16}(e),\quad \widetilde x=y+\widehat e.
\]
The corresponding elementwise output reconstruction is $\widetilde O_{32}=O_{16}+E_{16}$. The sums defining $\widetilde x$ and $\widetilde O_{32}$ are interpreted over the reals for the following bound. Assume round-to-nearest, no overflow or relevant underflow, and exact evaluation of the residual subtraction. With BF16 unit roundoff $u_{16}=2^{-8}$,
\begin{equation}
|\widetilde x-x|=|\widehat e-e|
\leq u_{16}|e|\leq u_{16}^2|x|,
\qquad |y-x|\leq u_{16}|x|.
\label{tpfapp:eq:residual-bound}
\end{equation}
Rounded subtraction contributes its own term, and explicit FP32 reconstruction contributes at most $u_{32}|\widetilde x|$, where $u_{32}=2^{-24}$, under the same normal-range assumptions. For an exactly evaluated row dot product, Equation~(\ref{tpfapp:eq:residual-bound}) gives
\[
\bigl|\langle dO,\widetilde O_{32}-O_{32}\rangle\bigr|
\leq u_{16}^2\sum_d|dO_d(O_{32})_d|.
\]
To relate this component to the computed row correction, let $O^*=PV$ be the exact mathematical output and decompose
\begin{align}
\widehat\delta-\langle dO,O^*\rangle
={}&\underbrace{\widehat\delta-\langle dO,\widetilde O_{32}\rangle}
                 _{\text{evaluation rounding}}\\
 &+\underbrace{\langle dO,\widetilde O_{32}-O_{32}\rangle}
                 _{\text{output conversion}}\nonumber\\
 &+\underbrace{\langle dO,O_{32}-O^*\rangle}
                 _{\text{forward arithmetic}}.\nonumber
\end{align}
Residual retention reduces the bound on the output-conversion term from first to second order in BF16 unit roundoff under the stated assumptions. This does not guarantee a lower error in every row correction: reconstruction errors can cancel in the dot product, and the other two terms remain. The design preserves the public BF16 output and supplies the correction without another probability recomputation, at the cost of one saved BF16 tensor.

\section{Workspace Complexity and Training Memory}
\label{tpfapp:memory}

TurboPairFormer stores BF16 gradient partials to reuse each recomputed tile and to fix the accumulation order. The query partials form one buffer of logical shape $[J_K,BN_o,H,N,D]$; with two bytes per element and $N_o=N$,
\begin{equation}
W_Q=2J_KBN_oHND=2BHDN^2\left\lceil\frac{N}{\tau_k}\right\rceil
=\Theta\!\left(\frac{BHDN^3}{\tau_k}\right)
\label{tpfapp:eq:wq}
\end{equation}
for a fixed tile width. Storing the bias partials of all outer groups would take $W_{\mathrm{bias,all}}=2BLHN^2$ bytes; the planner instead keeps $B_c$ examples, $R$ outer groups, and $Q_s$ query rows resident,
\begin{equation}
W_{\mathrm{bias,res}}=2B_cRHQ_sN,
\label{tpfapp:eq:wb}
\end{equation}
and reuses these buffers in successive launches. The saved output residual occupies $W_E=2BN_oHND$ bytes per attention call. A complete FP32 score gradient would occupy $W_{dS}=4BHN^3$ bytes, so when $\tau_k$ divides $N$,
\begin{equation}
\frac{W_Q}{W_{dS}}=\frac{D}{2\tau_k},
\label{tpfapp:eq:cubic}
\end{equation}
which is one quarter at $D=32$ and $\tau_k=64$. Table~\ref{tpfapp:tab:memory} lists these sizes. At crop 768, the selected plan uses $Q_s=768$ and $R=56$, keeping 252\,MiB of bias partials resident beside 1728\,MiB of query partials and a 144\,MiB residual per call.

\begin{table}[!htbp]
\centering
\caption{Analytical allocation sizes in MiB for $B=1$, $N_o=N$, $H=4$, $D=32$, $\tau_k=64$, $G=4$, and BF16 partials. These are tensor sizes, not measured training peaks; crop 1024 is extrapolated.}
\label{tpfapp:tab:memory}
\small
\setlength{\tabcolsep}{5pt}
\begin{tabular}{@{}rrrrrr@{}}
\toprule
$N$ & $W_Q$ & $W_{\mathrm{bias,all}}$ & $W_{\mathrm{bias,res}}$ & $W_E$ & $W_{dS}$ \\
\midrule
384 & 216 & 108 & 108 & 36 & 864 \\
640 & 1000 & 500 & 500 & 100 & 4000 \\
768 & 1728 & 864 & 252 & 144 & 6912 \\
1024 & 4096 & 2048 & 256 & 256 & 16384 \\
\bottomrule
\end{tabular}
\end{table}

Whether this workspace raises the training peak depends on its lifetime. Let $A(t)$ be the bytes of all other live tensors and $W(t)$ the temporary workspace at execution point $t$, and let $P_0=\max_tA(t)$ and $P_1=\max_t\{A(t)+W(t)\}$. Then
\begin{equation}
P_1=P_0
\quad\Longleftrightarrow\quad
W(t)\leq P_0-A(t)\ \text{for every }t,
\label{tpfapp:eq:headroom}
\end{equation}
because the right-hand condition states $A(t)+W(t)\leq P_0$, while $W(t)\geq0$ gives $P_1\geq P_0$. Transient workspace can therefore fit below a peak set elsewhere in the step by saved activations, gradients, or optimizer states. At crop 768, sweeping the bias-workspace target from 64 to 512\,MiB left the measured peak memory unchanged. The sweep holds the query-partial buffer fixed, so it shows only that bias scheduling does not raise the peak. The main limitation is this buffer, which grows cubically with crop size and linearly with local batch size; reducing it while keeping the fixed accumulation order is future work.

\section{Comparison with Existing Kernels}
\label{tpfapp:comparison}

Triangular-attention backward kernels differ in how they share recomputed probabilities and combine contributions across thread blocks. Table~\ref{tpfapp:tab:comparison} summarizes the inspected implementations.\footnote{Inspected sources: \href{https://github.com/latkins/trifast/blob/2d5251135c774d0585e9bc042d7209c6b7c470cf/src/trifast/triton.py\#L143-L501}{TriFast}, \href{https://github.com/bytedance/Protenix/blob/4c355be4553512f72453ecbfb65e69f4c35d1413/protenix/model/tri_attention/backward.py\#L519-L609}{Protenix}, \href{https://github.com/Supercomputing-System-AI-Lab/MegaFold/blob/845905e7c707cf4004404354156df97f2efc4fde/megafold/model/FusedEvoAttention/evoattention.py\#L352-L516}{MegaFold}, and the \href{https://github.com/aqlaboratory/openfold-3/blob/57059a86cdf87379914817bce8e3e84207995ae4/openfold3/core/kernels/triton/evoformer.py\#L753-L774}{newer} and \href{https://github.com/aqlaboratory/openfold-3/blob/4e2377bf73f664e05dbb68cf62e4b8774f0d899e/openfold3/core/kernels/triton/evoformer.py\#L597-L603}{earlier} OpenFold3 kernels.} Among them, only TurboPairFormer combines one recomputation, no dense FP32 score gradient, no floating-point atomics, and a fixed accumulation order.

\begin{table}[!htbp]
\centering
\caption{Structure of the inspected backward kernels. The three ``No'' columns indicate the absence of dense FP32 score-gradient storage or floating-point atomic accumulation. ``Fixed'' denotes an accumulation order independent of block retirement under fixed execution settings. Partial storage in TurboPairFormer remains cubic in $N$.}
\label{tpfapp:tab:comparison}
\small
\setlength{\tabcolsep}{3pt}
\begin{tabular}{@{}lcccccc@{}}
\toprule
Backend & Auto- & $P$ & No dense & No $dQ$ & No bias & Fixed \\
 & tune & recomputations & FP32 $dS$ & atomics & atomics & order \\
\midrule
TurboPairFormer & No & 1 & Yes & Yes & Yes & Yes \\
TriFast & Yes & 3 & Yes & Yes & Yes & Yes \\
Protenix & Yes & 3 & Yes & Yes & Yes & Yes \\
MegaFold & No & 2 & Yes & Yes & No & No \\
OpenFold3 Triton (newer) & No & 2 & Yes & No & Yes & No \\
OpenFold3 Triton (earlier) & No & 2 & Yes & Yes & No & No \\
ScaleFold & No & 1 & No & Yes & Yes & Yes \\
DeepSpeed-Evo & No & 1 & Yes & Yes & No & No \\
\bottomrule
\end{tabular}
\end{table}

\paragraph{Separate recomputation passes.} TriFast and Protenix compute $dQ$, $(dK,dV)$, and the pair-bias gradient in three kernels, each recomputing the probabilities from the saved softmax statistics. Every gradient element has one writer and no atomics are used, at the cost of three recomputations.

\paragraph{Dense score-gradient storage.} The ScaleFold path in Apex OpenFold Triton recomputes the probabilities once, writes the complete FP32 score gradient, and reduces the pair-bias gradient from it in a subsequent kernel. This avoids both repeated recomputation and atomics, but stores $W_{dS}=4BHN^3$ bytes (Appendix~\ref{tpfapp:memory}).

\paragraph{Atomic accumulation.} MegaFold recomputes twice, in a $(dK,dV)$ kernel and a fused $(dQ,d\,\mathrm{bias})$ kernel, and accumulates the bias gradient across outer positions with an atomic addition, so only that gradient can vary across replays. The newer OpenFold3 Triton source also recomputes twice but accumulates $dQ$ with atomics; the earlier source, used for the synthetic accuracy and replay results, accumulates the shared bias gradient with FP32 atomics. DeepSpeed-Evo, the DS4Sci Evoformer attention in DeepSpeed,\footnote{Official DS4Sci documentation: \href{https://www.deepspeed.ai/tutorials/ds4sci_evoformerattention/}{\texttt{deepspeed.ai/tutorials/ds4sci\_evoformerattention}}.} recomputes once in a fused backward kernel in which a block owns one slice and head and accumulates $dQ$ without atomics, but it combines the bias gradient across slices with atomic additions whose order depends on block execution.

\paragraph{Ordered partials.} Like DeepSpeed-Evo and ScaleFold, TurboPairFormer recomputes each probability tile once for the four gradients, but it differs from both in three ways. It assigns key tiles rather than whole slices to thread blocks, which exposes more parallelism at the cost of a query-partial reduction. It sums the shared pair-bias gradient through a fixed tree over four slices and ordered BF16 partials with FP32 running sums (Equations~(\ref{tpfapp:eq:partials}) and~(\ref{tpfapp:eq:ordered-fold})), instead of global atomics or a dense FP32 score gradient. It also forms the softmax correction from the BF16 output plus a saved residual, which neither of the two retains (Section~\ref{tpfapp:residual}). Its query-partial buffer is $D/(2\tau_k)$ of a dense FP32 score gradient, one quarter at the reported configuration (Equation~(\ref{tpfapp:eq:cubic})), but it remains cubic in $N$.

\section{Additional End-to-End Results}
\label{tpfapp:measurement}
\label{sec:tpf-supp-results}

This appendix reports a supplemental eight-H200 training comparison, training-loss behavior, additional accuracy and replay results, and fine-tuning validation details. The synthetic tests use absolute RMSE, whereas the captured-call studies use relative $L_2$ error.

\subsection{Training Step Latency on Eight H200 GPUs}
\label{tpfapp:measurement-results}
\label{tpfapp:h200-training}
\label{sec:tpf-supp-performance}

Table~\ref{tpfapp:tab:step-times} reports the supplemental comparison on eight H200 GPUs; each value is the mean GPU computation time over steps 21--400 of a 400-step run.

\begin{table}[!htbp]
\centering
\caption{Mean seconds per optimizer step on eight NVIDIA H200 GPUs, excluding 20 warmup steps from each 400-step run. Protenix, TriFast, MegaFold, and ScaleFold use cuEquivariance triangle multiplication. Bold denotes the lowest mean in each column.}
\label{tpfapp:tab:step-times}
\small
\begin{tabular}{@{}lrrr@{}}
\toprule
Backend & Crop 384 & Crop 640 & Crop 768 \\
\midrule
TurboPairFormer & \textbf{7.01} & \textbf{11.17} & \textbf{14.67} \\
Protenix & 7.35 & 11.30 & 14.92 \\
cuEquivariance & 7.43 & 12.11 & 16.51 \\
TriFast & 7.51 & 11.44 & 15.14 \\
MegaFold & 7.63 & 13.11 & 18.02 \\
DeepSpeed-Evo & 7.52 & 13.34 & 18.42 \\
ScaleFold & 7.91 & 12.58 & 17.15 \\
OpenFold3 Triton & 8.02 & 15.22 & 22.53 \\
\bottomrule
\end{tabular}
\end{table}

TurboPairFormer has the lowest mean at every crop. Its step time is 12.6\%, 26.6\%, and 34.9\% lower than OpenFold3 Triton's and 6.8\%, 16.3\%, and 20.4\% lower than DeepSpeed-Evo's at crops 384, 640, and 768. At crop 768, it is lower than cuEquivariance's, ScaleFold's, Protenix's, and TriFast's by 11.1\%, 14.5\%, 1.7\%, and 3.1\%. Pairing the 380 matched steps at crop 640 gives reductions of $1.18\%$ $[0.89,1.53]$ and $2.31\%$ $[1.98,2.66]$ relative to Protenix and TriFast (95\% moving-block bootstrap intervals); these intervals cover variation over matched inputs within one run, not across repeated runs. These integrated configurations do not isolate the effect of attention (Table~\ref{tpfapp:tab:component-ablation}).

\subsection{Training Loss}
\label{tpfapp:measurement-stability}
\label{sec:tpf-supp-stability}

In the H200 runs at crops 384 and 640, every loss remains finite, and the final-50-step mean losses of the seven original configurations span 1.31650--1.31981 and 1.35317--1.35514, which shows agreement over this interval, not equal losses.

\begin{figure}[!htbp]
\centering
\includegraphics[width=0.85\linewidth,trim=0 4 0 1,clip]{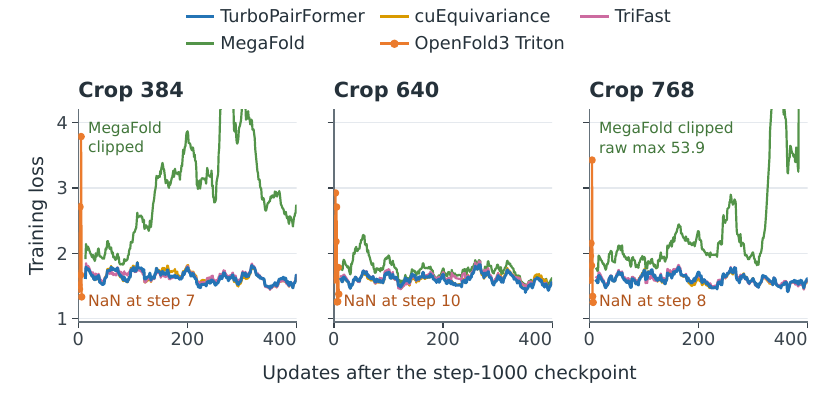}
\caption{\textbf{Training stability after a shared 1,000-step learning-rate warmup.} Five backend configurations resume the same checkpoint on one H100, with the same restart seed and matched sample order within each crop. Curves show trailing 25-update mean losses, except OpenFold3 Triton, whose finite raw losses are shown before the first NaN. MegaFold curves at crops 384 and 768 exceed the displayed range; 53.9 is the maximum raw loss at crop 768.}
\label{fig:stability}
\end{figure}

\paragraph{Post-warmup protocol.}
\label{tpfapp:stability-protocol}
Table~\ref{tpfapp:tab:training-settings} lists the settings of the runs in Figure~\ref{fig:stability}. Each run restores the model, EMA, optimizer, scheduler, training-loop, and data-module states of the shared checkpoint (Section~\ref{sec:stability}) and continues on the post-warmup plateau of OpenFold3's learning-rate schedule. The runs use Adam (betas 0.9 and 0.95, epsilon $10^{-8}$), gradient clipping at 10, BF16 mixed precision, batch size one, a model with 48 Pairformer blocks, and a fixed PDB training subset cropped contiguously, spatially, or at an interface with weights 0.2, 0.4, and 0.4. Matching sample order does not imply identical random draws throughout training. The backends are cuEquivariance 0.10.0, TriFast 0.1.13, MegaFold, OpenFold3 Triton with Triton 3.6.0, and an archived August 2026 TurboPairFormer build with single-recomputation backward and the output residual. After 400 updates, TriFast's mean losses are 1.648, 1.625, and 1.609, which, with the means of Section~\ref{sec:stability}, gives a within-crop spread of at most about 0.022; MegaFold's are 2.962, 1.746, and 2.680.

\begin{table}[!htbp]
\centering
\caption{Settings of the post-warmup study (Figure~\ref{fig:stability}) and of fine-tuning (Section~\ref{sec:tpf-validation}).}
\label{tpfapp:tab:training-settings}
\small
\begin{tabular}{@{}lll@{}}
\toprule
Setting & Post-warmup study & Fine-tuning validation \\
\midrule
Starting weights & Preview2 after 1,000 updates at crop 640 & Preview2 \\
GPUs & One H100 & 16 H100 (two nodes) \\
Crop size & 384, 640, 768 & 384 \\
Training updates & 400 & 1,000 \\
Learning rate & $1.8\times10^{-3}$, post-warmup plateau & $3\times10^{-4}$, 50 warmup steps \\
EMA decay & 0.999 & 0.98 \\
Diffusion samples & 48 & 32 \\
Seed & 42 & 42 \\
\bottomrule
\end{tabular}
\end{table}

\subsection{Additional Accuracy Results}
\label{sec:tpf-supp-protocol}
\label{sec:tpf-supp-original-plots}

\paragraph{Synthetic inputs.}
On one H100, we evaluate $N\in\{384,640,768\}$ and $D\in\{16,32\}$ with ten seeds per shape, batch size one, four heads, and $N$ slices. Each entry of $Q$, $K$, and $V$ is standard normal plus, with probability 0.001, a normal outlier with standard deviation 10; $dO$ is standard normal. All are rounded once to BF16, and the bias and mask are zero. We report the absolute RMSE of the public BF16 gradients against an FP64 reference evaluated on the same quantized inputs, as the mean and sample standard deviation over seeds.

\paragraph{Implementation variants.}
\label{tpfapp:synthetic-versions}
The accuracy and replay studies use TurboPairFormer with the residual, Protenix, cuEquivariance 0.10.0, the earlier OpenFold3 Triton source (Appendix~\ref{tpfapp:comparison}), TriFast 0.1.13, MegaFold, ScaleFold, and DeepSpeed-Evo from DeepSpeed 0.19.8. TriFast, MegaFold, and DeepSpeed-Evo include our numerical modifications: TriFast and MegaFold save the FP32 row maximum and exponential sum separately for backward probability recomputation, and DeepSpeed-Evo centers scores in FP32 before exponentiation, accepts the FP32 additive mask, and has corrected stream and boundary handling. ScaleFold runs through an adapter for these shapes.

\paragraph{Residual ablation.}
\label{sec:tpf-supp-residual}
Figure~\ref{tpfapp:fig:residual-on-off} shows the residual ablation summarized in Section~\ref{sec:accuracy_determinism}. Its dashed lines give the RMSE of the FP64 reference rounded once to BF16, the smallest RMSE any BF16 tensor can attain. With the residual, the mean $dQ$ error falls from 2.76--3.37 times this floor to 1.76--1.78 times, and the mean $dK$ and $d\,\mathrm{bias}$ errors from 1.98--2.50 times to 1.42--1.43 times. On captured inputs, the corresponding OpenFold3 median errors decrease by 35.9--55.5\%, at an additional forward-plus-backward cost of 1.06--2.33\% across five H100 operator cases.

\begin{figure}[!htbp]
\centering
\includegraphics[width=\linewidth]{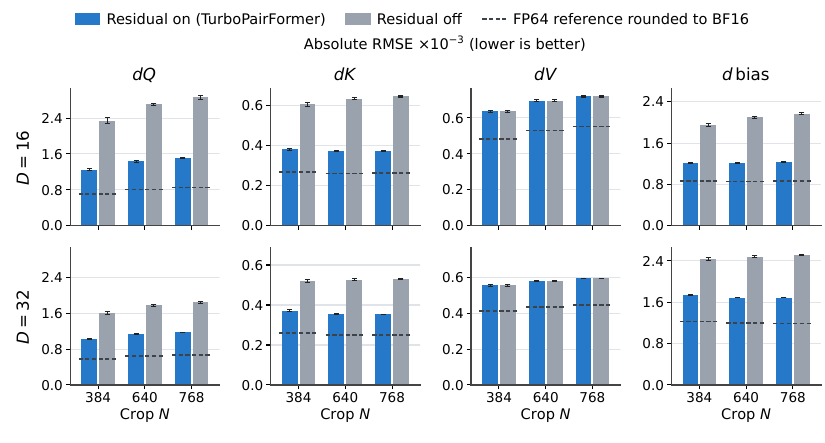}
\caption{\textbf{Residual compensation on and off.} TurboPairFormer runs with and without the saved output residual on one H100, with identical inputs and incoming gradients. Bars show the mean absolute RMSE of the public BF16 gradients against an FP64 reference over ten seeds, and error bars the sample standard deviation. Dashed lines mark the RMSE of the FP64 reference rounded once to BF16. The residual enters only the softmax correction, so $dV$ is bitwise identical in both settings.}
\label{tpfapp:fig:residual-on-off}
\end{figure}

\paragraph{Captured training inputs.}
\label{sec:tpf-supp-captured}
Table~\ref{tab:tpf-fp64-accuracy} reports relative $L_2$ errors over 32 attention calls captured during OpenFold3 training. TurboPairFormer's median $dQ$ error is 0.3769\%, about 54\% lower than either comparator's; $dK$ and the bias gradient also improve, while the output and $dV$ remain close.

\begin{table}[!htbp]
\centering
\caption{Median relative $L_2$ error (\%) against an FP64 reference over 32 captured OpenFold3 attention calls. Output error is evaluated on valid query rows. Bold marks the lowest value in each row.}
\label{tab:tpf-fp64-accuracy}
\label{tab:tpf-supp-captured-fp64}
\small
\begin{tabular}{@{}lrrr@{}}
\toprule
Quantity & Protenix & cuEquivariance & TurboPairFormer \\
\midrule
$O$ & 0.1695 & \textbf{0.1686} & \textbf{0.1686} \\
$dQ$ & 0.8202 & 0.8222 & \textbf{0.3769} \\
$dK$ & 0.4347 & 0.4358 & \textbf{0.2437} \\
$dV$ & 0.2201 & \textbf{0.2164} & \textbf{0.2164} \\
$d\,\mathrm{bias}$ & 0.3534 & 0.3506 & \textbf{0.2449} \\
\bottomrule
\end{tabular}
\end{table}

\subsection{Deterministic Fixed-Input Replay}
\label{tpfapp:measurement-replay}
\label{sec:tpf-supp-replay}

On one H100, we test 100 seeds for each of the six shapes, 600 cases per method, with the input recipe of Section~\ref{sec:tpf-supp-protocol}. Each case runs five forward and backward calls in one process with fixed BF16 inputs and $dO$, and passes when all four gradients are finite and bitwise identical across the five calls. TurboPairFormer, Protenix, ScaleFold, and TriFast pass all 600 cases; MegaFold passes 165, while cuEquivariance, OpenFold3 Triton, and DeepSpeed-Evo pass none (Figure~\ref{fig:tpf-determinism}). These measurements assess operator repeatability, not end-to-end training determinism, since optimization can amplify small numerical perturbations~\citep{summers2021nondeterminism}.

\begin{figure}[!htbp]
\centering
\includegraphics[width=\linewidth]{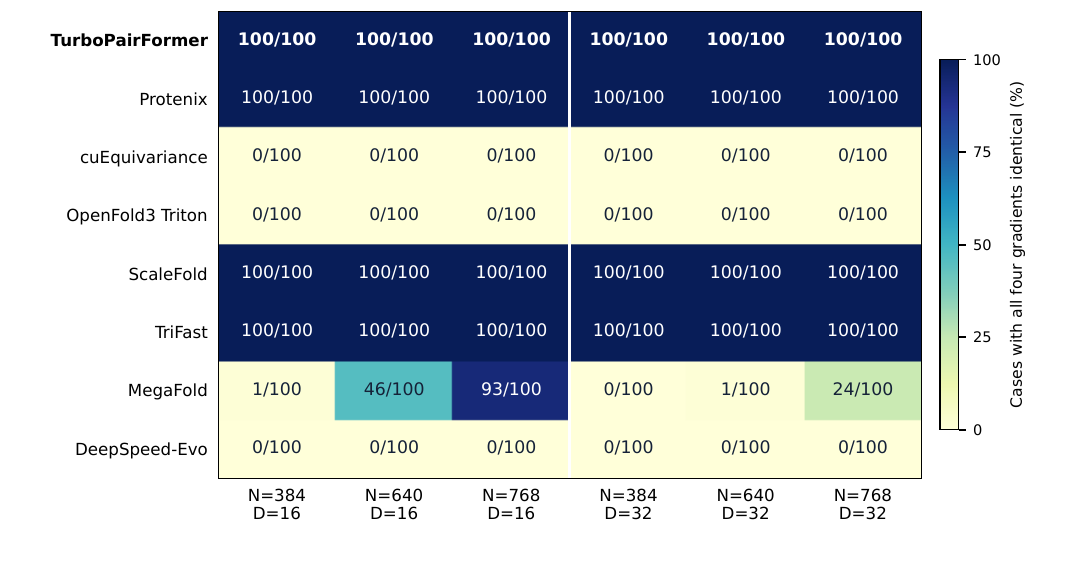}
\caption{\textbf{Bitwise gradient replay across eight attention implementations.} Each shape contains 100 fixed synthetic inputs, with five calls per input on one H100. A case passes only when all four public BF16 gradients are finite and bitwise identical across all five calls; forward output is excluded from this joint criterion.}
\label{fig:tpf-determinism}
\end{figure}

\subsection{Sixteen-GPU Fine-Tuning Validation}
\label{sec:tpf-validation-details}

Table~\ref{tpfapp:tab:training-settings} lists the fine-tuning settings. Each final EMA checkpoint and the Preview2 starting point~\citep{openfold3preview2} are evaluated with OpenFold3's native validation, using five rollouts per target and averaging, for each metric, the top-1 value (the rollout with the lowest predicted distance error) and the best-of-five value. All checkpoints use the same OpenFold3 Triton evaluation kernels and, for a given evaluation seed, the same diffusion noise, so seeds 0--4 measure evaluation sampling variability rather than variation across training runs. The model-selection score weights the intra-protein, intra-ligand, and protein--protein lDDT~\citep{mariani2013lddt} by 20 and the protein--ligand lDDT and the relative accessible surface area (RASA) term by 10, normalized by 80. Of the 85 complexes, 11, 77, 66, and 12 are eligible for the intra-protein, protein--ligand, intra-ligand, and protein--protein lDDT, 12 for weighted DockQ~\citep{basu2016dockq}, 83 for RASA, and all 85 for GDT-TS~\citep{zemla2003lga} and RMSD. Exclusion from the training split and from Preview2 pretraining was not verified, so the cohort is not an established holdout.

\paragraph{Results.} The pooled standard deviation over seeds is 0.0045 for model selection, larger than the 0.0029 standard deviation of the nine fine-tuned means. Table~\ref{tpfapp:tab:validation-paired} averages each target over the five seeds and bootstraps the 85 targets 2,000 times. A second TurboPairFormer run with identical code, configuration, and seed scores 0.0057 higher under seeds 0--2, which exceeds the magnitude of seven of the eight TurboPairFormer--comparator differences; two runs cannot characterize training variability across backends.

\begin{table}[!htbp]
\centering
\caption{Paired model-selection difference, TurboPairFormer minus each checkpoint, with targets averaged over the five seeds and 95\% bootstrap intervals over the 85 targets.}
\label{tpfapp:tab:validation-paired}
\small
\begin{tabular}{@{}lrr@{}}
\toprule
Checkpoint & Difference & 95\% interval \\
\midrule
Preview2 (step 0) & $-0.0076$ & $[-0.0236, +0.0052]$ \\
PyTorch & $-0.0030$ & $[-0.0092, +0.0022]$ \\
OpenFold3 Triton & $+0.0012$ & $[-0.0036, +0.0057]$ \\
cuEquivariance & $-0.0035$ & $[-0.0092, +0.0024]$ \\
Protenix & $-0.0009$ & $[-0.0083, +0.0061]$ \\
TriFast & $-0.0030$ & $[-0.0094, +0.0032]$ \\
MegaFold & $-0.0082$ & $[-0.0169, -0.0002]$ \\
DeepSpeed-Evo & $-0.0006$ & $[-0.0077, +0.0055]$ \\
ScaleFold & $-0.0050$ & $[-0.0131, +0.0019]$ \\
\bottomrule
\end{tabular}
\end{table}

\section{Single-GPU Operator Measurements}
\label{tpfapp:operator}

This appendix times single operator calls on one NVIDIA H100. Both operators are only part of a training step, so these reductions do not carry over directly to step time (Table~\ref{tpfapp:tab:component-ablation}).

\subsection{Triangular Attention}
\label{tpfapp:operator-attention}

We time one attention call at $N\in\{384,640,768\}$ and $D\in\{16,32\}$, with $B=1$, $N_o=N$, and $H=4$; OpenFold3 uses $D=16$ only in its two-block template stack. The backends are TurboPairFormer in its training configuration, cuEquivariance 0.10.0, and Protenix after autotuning, each called as the OpenFold3 attention layer calls it, with gradients taken with respect to the projection outputs. Each entry in Table~\ref{tpfapp:tab:operator} averages two processes, each pooling five rounds of 15 calls with the backend order rotated between rounds.

\begin{table}[!htbp]
\centering
\caption{Time in ms of one attention call on one H100 ($B=1$, $N_o=N$, $H=4$) for TurboPairFormer (Ours) in its training configuration, cuEquivariance 0.10.0 (cuEq), and Protenix after autotuning. The last three columns give TurboPairFormer's change in time relative to cuEquivariance; negative values mean TurboPairFormer is faster.}
\label{tpfapp:tab:operator}
\footnotesize
\setlength{\tabcolsep}{3pt}
\begin{tabular}{@{}rrrrrrrrrrrrrr@{}}
\toprule
 & & \multicolumn{3}{c}{Forward} & \multicolumn{3}{c}{Backward} & \multicolumn{3}{c}{Forward + backward} & \multicolumn{3}{c}{Change vs.\ cuEq (\%)} \\
\cmidrule(lr){3-5}\cmidrule(lr){6-8}\cmidrule(lr){9-11}\cmidrule(l){12-14}
$D$ & $N$ & Ours & cuEq & Protenix & Ours & cuEq & Protenix & Ours & cuEq & Protenix & Fwd & Bwd & F+B \\
\midrule
16 & 384 & 0.384 & 0.640 & 0.417 & 1.167 & 2.097 & 1.285 & 1.493 & 2.623 & 1.587 & $-39.9$ & $-44.3$ & $-43.1$ \\
16 & 640 & 1.177 & 2.041 & 1.316 & 4.300 & 8.046 & 4.655 & 5.398 & 9.918 & 5.830 & $-42.3$ & $-46.6$ & $-45.6$ \\
16 & 768 & 1.870 & 3.341 & 2.137 & 7.392 & 13.244 & 7.563 & 9.025 & 16.478 & 9.499 & $-44.0$ & $-44.2$ & $-45.2$ \\
\midrule
32 & 384 & 0.474 & 0.708 & 0.502 & 1.520 & 2.430 & 1.482 & 1.920 & 3.015 & 1.854 & $-33.1$ & $-37.5$ & $-36.3$ \\
32 & 640 & 1.538 & 2.192 & 1.602 & 5.468 & 8.764 & 5.654 & 6.810 & 10.823 & 7.129 & $-29.8$ & $-37.6$ & $-37.1$ \\
32 & 768 & 2.448 & 3.557 & 2.630 & 9.348 & 14.544 & 8.764 & 11.578 & 17.965 & 11.262 & $-31.2$ & $-35.7$ & $-35.6$ \\
\bottomrule
\end{tabular}
\end{table}

TurboPairFormer is faster than cuEquivariance at every shape and in every mode: at $D=32$, its forward, backward, and forward-plus-backward times are 29.8--33.1\%, 35.7--37.6\%, and 35.6--37.1\% lower, and at $D=16$ they are 39.9--44.0\%, 44.2--46.6\%, and 43.1--45.6\% lower. Against Protenix, its forward pass is 3.9--12.5\% faster at every shape. Its backward pass is 2.3--9.2\% faster at $D=16$; at $D=32$, it is 3.3\% faster at $N=640$ but 2.5\% and 6.7\% slower at $N=384$ and $768$, so forward plus backward ranges from 7.4\% faster to 3.6\% slower. The saved residual adds 1.0--2.6\% to forward-plus-backward time. TurboPairFormer's gradients differ from those of the other two backends by a relative $L_2$ distance of at most $4.4\times10^{-3}$.

\paragraph{One recomputation versus two.} As an ablation, a two-pass backward variant that recomputes the probabilities once for the key and value gradients and again for the query and pair-bias gradients takes 28.3--64.0\% longer than the single-recomputation backward, with the output residual disabled in both. Its kernels also differ in other respects, so the gap does not isolate the recomputation count. One call of our kernel allocates 0.36--3.09\,GiB beyond its inputs, compared with 0.19--1.46\,GiB for cuEquivariance, 0.10--0.73\,GiB for Protenix, and 0.24--1.23\,GiB for the two-pass variant, which stores no query partials. At $D=32$ and $N=768$, the query partials take 1.69 of the 3.09\,GiB.

\subsection{Forward Dataflow Ablation}
\label{tpfapp:forward-ablation}

Forward executions outnumber backward executions in OpenFold3 training: in a profiled crop-768 update with three recycling iterations, 600 forward executions of triangular attention, 120 of them repeated by activation checkpointing, accompany 120 backward executions. A four-arm ablation isolates the forward changes of Appendix~\ref{tpfapp:forward} at $D=32$. The baseline stages keys and values with TMA but reads the bias and mask directly from global memory and evaluates all probabilities of a key tile before aggregation. The core arm adds bias staging, mask prefetching, and chunked evaluation; the pack arm changes only the repacking of $Q$, $K$, and $V$; and the combined arm contains both changes. All arms share the same backward kernels and give bitwise-identical outputs and gradients.

\begin{table}[!htbp]
\centering
\caption{Forward dataflow ablation on one H100 ($B=1$, $N_o=N$, $H=4$, $D=32$). Times are means of 160 forward calls per arm; percentages are mean reductions over eight paired blocks relative to the baseline.}
\label{tpfapp:tab:forward-ablation}
\small
\begin{tabular}{@{}rrrrrrrr@{}}
\toprule
 & \multicolumn{3}{c}{Forward time (ms)} & \multicolumn{2}{c}{Forward reduction (\%)} & \multicolumn{2}{c}{F+B reduction (\%)} \\
\cmidrule(lr){2-4}\cmidrule(lr){5-6}\cmidrule(l){7-8}
$N$ & Baseline & Core & Combined & Core & Combined & Core & Combined \\
\midrule
384 & 0.547 & 0.408 & 0.388 & 25.5 & 29.1 & 8.1 & 9.2 \\
640 & 2.037 & 1.478 & 1.426 & 27.4 & 30.0 & 8.3 & 9.0 \\
768 & 3.336 & 2.419 & 2.341 & 27.5 & 29.8 & 7.8 & 8.5 \\
\bottomrule
\end{tabular}
\end{table}

The core changes cut forward time by 25.5--27.5\%, the combined arm by 29.1--30.0\%, and repacking alone by 2.4--3.5\%. The combined arm lowers forward-plus-backward time by 8.5--9.2\%, while backward time changes by at most 0.2\%. The core changes are measured together, so the ablation does not apportion their saving among bias staging, mask prefetching, and chunked evaluation.

\subsection{Triangle Multiplication}
\label{tpfapp:operator-trimul}

We time one forward-plus-backward triangle multiplicative update against cuEquivariance 0.10.0 at the twelve OpenFold3 training shapes with one example: crops 384, 640, and 768; 128 channels, as in the Pairformer and MSA module, or 64, as in the template stack; and the incoming and outgoing directions. Each call runs under BF16 autocast, and each of three rounds times cuEquivariance, TurboPairFormer, and cuEquivariance again.

\begin{table}[!htbp]
\centering
\caption{Median time in ms of one triangle multiplicative update, forward plus backward, on one H100 with one example. Parentheses give TurboPairFormer's (Ours) change in time relative to cuEquivariance 0.10.0; negative values mean TurboPairFormer is faster.}
\label{tpfapp:tab:trimul-operator}
\small
\begin{tabular}{@{}rrrrrr@{}}
\toprule
 & & \multicolumn{2}{c}{Incoming} & \multicolumn{2}{c}{Outgoing} \\
\cmidrule(lr){3-4}\cmidrule(l){5-6}
$N$ & Channels & cuEquivariance & Ours & cuEquivariance & Ours \\
\midrule
384 & 128 & 2.769 & 1.326 ($-52.1\%$) & 2.777 & 1.323 ($-52.4\%$) \\
640 & 128 & 3.615 & 3.334 ($-7.8\%$) & 3.678 & 3.336 ($-9.3\%$) \\
768 & 128 & 5.142 & 4.788 ($-6.9\%$) & 5.210 & 4.782 ($-8.2\%$) \\
\midrule
384 & 64 & 2.734 & 1.111 ($-59.4\%$) & 2.735 & 1.116 ($-59.2\%$) \\
640 & 64 & 2.777 & 1.921 ($-30.8\%$) & 2.784 & 1.922 ($-31.0\%$) \\
768 & 64 & 2.842 & 2.692 ($-5.3\%$) & 2.879 & 2.691 ($-6.5\%$) \\
\bottomrule
\end{tabular}
\end{table}

TurboPairFormer is faster at all twelve shapes. With 128 channels, it takes 52.1--52.4\%, 7.8--9.3\%, and 6.9--8.2\% less time at $N=384$, 640, and 768. cuEquivariance takes 2.7--2.9\,ms at every shape except the 128-channel shapes with $N\geq640$, a per-call floor that explains the large reductions at $N=384$ and with 64 channels.

\section{Choosing the Deterministic Reduction}
\label{tpfapp:determinism}

With one recomputation, a thread block emits each tile's contributions to all four gradients, so any gradient that sums along an axis the block does not span needs a reduction across thread blocks. Under the work partitioning of Section~\ref{sec:single}, two remain: the query gradient over key tiles and the pair-bias gradient over outer groups. TurboPairFormer makes both deterministic with materialized partials summed in index order (Section~\ref{sec:leaves}). This appendix explains the choice and measures it on one H100. At $D=32$, the fastest tested non-deterministic variant of our kernel, which accumulates the pair-bias gradient with atomics, takes 2.8--4.1\% less backward time, and the deterministic kernel is faster than assigning one thread block per slice at every shape. At $D=16$, which only the template stack uses, the fastest tested non-deterministic variant takes 11.0--13.7\% less time.

\subsection{Which Gradients Need a Cross-Block Reduction}
\label{tpfapp:crossing}
\label{tpfapp:principles}

Kernels that traverse the scores once per gradient need no cross-block reduction, because each traversal can span the axis along which its gradient sums; TriFast and Protenix work this way at the price of three recomputations (Appendix~\ref{tpfapp:comparison}). With one recomputation, the work partitioning decides which gradients need such a reduction and how many blocks a launch contains (Table~\ref{tpfapp:tab:ownership}).

The pair-bias gradient stays within a thread block only if the block spans all $N_o$ outer indices (lower block of Table~\ref{tpfapp:tab:ownership}). A block that owns one query--key tile then leaves $dQ$, $dK$, and $dV$ to cross-block reductions, and a block that owns a key tile leaves too few blocks: $HJ_K=48$ at $N=768$ with four heads, fewer than the 132 SMs of an H100. We therefore partition the outer indices, a design choice that requires a cross-block pair-bias reduction; this reduction exists only because all slices share the bias. Among the partitionings that split the outer indices, assigning key tiles keeps $dK$ and $dV$ in registers and leaves $dQ$ as the only slice-shaped cross-block reduction, whereas assigning whole slices, as DeepSpeed-Evo does, removes that reduction but launches $J_K$ times fewer blocks. Grouping $G=4$ outer indices per block sums their pair-bias contributions on chip in a fixed tree, dividing the number of bias partials by $G$.

\begin{table}[!htbp]
\centering
\caption{Work partitioning for one example. A thread block handles one head, owns the listed part of the $(o,q,k)$ tile space, and iterates over the axes it spans. $J_Q$ and $J_K$ are the numbers of query and key tiles. The partials column counts the $Q$-sized partial tensors needed when the cross-block reductions of slice-shaped gradients use materialized partials; the upper block also needs $L=N_o/G$ bias-sized partials for the pair bias.}
\label{tpfapp:tab:ownership}
\small
\begin{tabular}{@{}lllll@{}}
\toprule
Thread block owns & Local & Cross-block & Partials & Thread blocks \\
\midrule
\multicolumn{5}{@{}l}{\emph{$G$ outer indices per thread block}} \\
\quad Key tile (ours) & $dK$, $dV$ & $dQ$, $d\,\mathrm{bias}$ & $J_K$ & $HJ_KN_o/G$ \\
\quad Query tile & $dQ$ & $dK$, $dV$, $d\,\mathrm{bias}$ & $2J_Q$ & $HJ_QN_o/G$ \\
\quad Whole slices (DeepSpeed-Evo) & $dQ$, $dK$, $dV$ & $d\,\mathrm{bias}$ & none & $HN_o/G$ \\
\midrule
\multicolumn{5}{@{}l}{\emph{All $N_o$ outer indices per thread block}} \\
\quad Query--key tile & $d\,\mathrm{bias}$ & $dQ$, $dK$, $dV$ & $J_K+2J_Q$ & $HJ_QJ_K$ \\
\quad Key tile & $dK$, $dV$, $d\,\mathrm{bias}$ & $dQ$ & $J_K$ & $HJ_K$ \\
\bottomrule
\end{tabular}
\end{table}

\subsection{Three Ways to Make a Cross-Block Reduction Deterministic}
\label{tpfapp:mechanisms}

A reduction across thread blocks is bitwise repeatable only if its association order is fixed or its arithmetic is order-independent, as with binned accumulators~\citep{ahrens2020reproducible} or superaccumulators~\citep{collange2015reproducible}. Attention kernels fix the order in three ways (Figure~\ref{fig:ordering}a).

\paragraph{Semaphore-ordered atomic accumulation.} Each thread block atomically adds its contribution to one FP32 buffer, and a semaphore grants the turns in a fixed order, so a block whose turn has not come waits. The deterministic mode of FlashAttention-3 accumulates the query gradient this way, in ascending key-tile order for non-causal attention~\citep{shah2024flashattention3}, and DASH chooses the turn order together with the tile schedule so that, in its scheduling model, the waits do not lengthen the critical path~\citep{qiang2026dash}.

\paragraph{One thread block per slice.} A thread block that owns whole slices visits every key tile itself and accumulates the query gradient in program order. DeepSpeed-Evo obtains a deterministic query gradient this way~\citep{song2023deepspeed4science}, and so does the deterministic mode of FlashAttention-2 for triangular attention. It keeps one FP32 query-gradient buffer per split and sets the number of splits to the SM count divided by the number of attention problems, rounded up,\footnote{\label{fn:fa2-source}Official FlashAttention-2 implementation, commit \texttt{98eb7998}: \href{https://github.com/Dao-AILab/flash-attention/blob/98eb7998a0eba4047c7a30375522569d4b8efb20/csrc/flash_attn/flash_api.cpp\#L921-L930}{FP32 buffers per split} and \href{https://github.com/Dao-AILab/flash-attention/blob/98eb7998a0eba4047c7a30375522569d4b8efb20/csrc/flash_attn/src/flash_bwd_launch_template.h\#L71-L124}{deterministic launch and reduction}.} which yields one split per slice for the $BN_oH$ problems of triangular attention (3,072 at $N=768$ with four heads). No block waits, but the launch has $J_K$ times fewer blocks than partitioning by key tile.

\paragraph{Materialized partials and a reduction kernel.} Each thread block writes its contribution as a partial that no other block writes, and after these blocks finish, a separate kernel sums the partials in index order. TurboPairFormer applies this to the query gradient, in the ascending key-tile order of FlashAttention-3's turns, and to the pair-bias gradient, a reduction that standard attention does not have. No block waits; the partials cost memory, one additional pass, and, because they are stored in BF16, one rounding per partial before the FP32 sum.

\paragraph{Choice for triangular attention.} Once the outer indices are partitioned, semaphore ordering would need, for every pair-bias tile, a chain of turns through all $L=N_o/G$ outer groups (192 at $N=768$, against $J_K=12$ key tiles for $dQ$), and a schedule that keeps both families of chains free of stalls. Materialized partials handle both reductions with one mechanism. Because each sum is ordered by partial index, they also let the planner split the backward pass into memory-bounded launches without changing the gradient bits (Proposition~\ref{tpfapp:prop:invariance}).

\begin{figure}[!t]
\centering
\includegraphics[width=\linewidth]{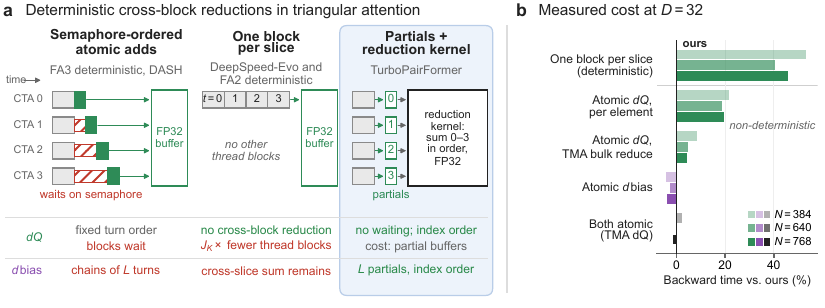}
\caption{\textbf{Making the cross-block reductions deterministic.} (a) Four thread blocks (CTA 0--3) contribute to one query-gradient reduction; the bottom row gives each scheme's consequence for the pair-bias gradient, which sums across slices. With semaphore-ordered atomic adds, the blocks add into one FP32 buffer in turns and wait for them; FlashAttention-3 grants the turns in key-tile order, and DASH chooses the order together with the schedule. With one thread block per slice, a single block visits all $J_K$ key tiles $t$ of the slice, so the launch has $J_K$ times fewer thread blocks; the deterministic mode of FlashAttention-2 falls in this class here. With materialized partials, blocks write one partial per key tile without waiting and a reduction kernel sums them in index order; TurboPairFormer uses this for $dQ$ ($J_K$ partials) and $d\,\mathrm{bias}$ ($L$ partials). (b) Change in backward time of variants of our kernel relative to the deterministic kernel at $D=32$ on one H100 (Table~\ref{tpfapp:tab:ordering-cost} and Section~\ref{tpfapp:slice-ownership}). The atomic variants are not deterministic; ``both atomic'' combines TMA bulk reduction for $dQ$ with atomic $d\,\mathrm{bias}$. Semaphore-ordered variants are not measured.}
\label{fig:ordering}
\end{figure}

\subsection{Measured Cost of the Reduction Implementations}
\label{tpfapp:ordering-cost}

We replace the deterministic reductions of our backward kernel with FP32 atomic accumulation and keep everything else fixed: atomic pair-bias accumulation; per-element atomic additions to $dQ$; or, as FlashAttention-3 accumulates the query gradient, $dQ$ tiles staged in shared memory and added with one TMA bulk reduction each. Each variant is timed in six rounds of 30 calls with the variant order rotated, and a second process reproduces every percentage within 0.7 percentage points.

\begin{table}[!htbp]
\centering
\caption{Backward time of one TurboPairFormer attention call on one H100 ($B=1$, $N_o=N$, $H=4$). The deterministic kernel reduces both gradients through materialized partials. Each other column replaces the named reductions with non-deterministic atomic accumulation in FP32 and gives the change in time relative to the deterministic kernel; positive values mean the variant is slower. ``Both'' combines atomic pair-bias accumulation with the named query-gradient accumulation.}
\label{tpfapp:tab:ordering-cost}
\small
\setlength{\tabcolsep}{4pt}
\begin{tabular}{@{}rrrrrrrr@{}}
\toprule
 & & & \multicolumn{5}{c}{Change with non-deterministic atomic accumulation (\%)} \\
\cmidrule(l){4-8}
$D$ & $N$ & \shortstack[r]{Deterministic\\(ms)} & \shortstack[r]{Atomic\\$d\,\mathrm{bias}$} & \shortstack[r]{Per-element\\$dQ$} & \shortstack[r]{Both,\\per-element} & \shortstack[r]{TMA bulk\\$dQ$} & \shortstack[r]{Both,\\TMA bulk} \\
\midrule
16 & 384 & 0.849 & $-8.0$ & $+1.9$ & $-6.7$ & $-5.2$ & $-11.0$ \\
16 & 640 & 3.582 & $-7.3$ & $-0.3$ & $-7.2$ & $-7.1$ & $-12.6$ \\
16 & 768 & 6.349 & $-7.9$ & $-0.1$ & $-7.8$ & $-7.2$ & $-13.7$ \\
\midrule
32 & 384 & 1.091 & $-4.1$ & $+21.5$ & $+13.6$ & $+8.3$ & $+2.2$ \\
32 & 640 & 4.505 & $-2.8$ & $+18.9$ & $+13.3$ & $+4.6$ & $-0.1$ \\
32 & 768 & 8.045 & $-3.9$ & $+19.6$ & $+13.0$ & $+4.2$ & $-1.3$ \\
\bottomrule
\end{tabular}
\end{table}

At $D=32$, the materialized query-gradient partials beat both atomic alternatives, which lengthen the backward pass by 18.9--21.5\% and 4.2--8.3\%, and atomic pair-bias accumulation alone is the fastest tested non-deterministic variant, 2.8--4.1\% faster. At $D=16$, combining TMA bulk reduction of $dQ$ with atomic pair-bias accumulation is fastest, 11.0--13.7\% faster. Profiles at $N=768$ and $D=32$ trace the slowdown of per-element atomics to contention in the L2 cache, which receives 99.1 million single-sector reduction requests, whereas the reduction kernel reads the 1.69\,GiB of query partials once in 0.63\,ms under profiler clocks. This compares BF16 partials and a fixed FP32 reduction with direct FP32 atomics, so it does not isolate the cost of determinism, and it does not bound semaphore-ordered accumulation, which requires direct measurement.

\subsection{One Thread Block per Slice}
\label{tpfapp:slice-ownership}

Assigning each thread block the whole slices of its four outer indices removes the query-gradient reduction, as in DeepSpeed-Evo: the block visits every key tile itself and accumulates $dQ$ in a dense FP32 buffer. We built this as a variant of our kernel with the same pair-bias path and timed it against the key-tile kernel of the same build, under the default workspace plan and under a plan with every outer group resident, which the variant can afford because it stores no query partials. It is deterministic, and its key, value, and pair-bias gradients equal ours bit for bit, but its backward pass is 22.8--95.9\% slower across the six shapes and two plans. The gap narrows when more outer groups run concurrently, from 45.9\% to 22.8\% at $D=32$ and $N=768$, which points to its $J_K$ times smaller launch.

\subsection{Workspace Plans, Memory, and Precision}
\label{tpfapp:plan-precision}

At every shape, our kernel's four gradients are bitwise identical under the default workspace plan and under the plan with every outer group resident, as Proposition~\ref{tpfapp:prop:invariance} predicts; at $N=768$, these plans process the 192 outer groups in four launches and in one. The resident plan shortens the backward pass by 5.1\% there but needs an 864\,MiB rather than a 252\,MiB bias workspace. The BF16 partials are the main memory cost: 1.69\,GiB of query partials at $D=32$ and $N=768$, against one 288\,MiB FP32 buffer for semaphore-ordered or per-slice accumulation. Relative to FP32 atomic accumulation, their BF16 rounding changes $dQ$ and $d\,\mathrm{bias}$ by a relative $L_2$ error of $2.7$--$2.8\times10^{-3}$ and $1.6$--$1.7\times10^{-3}$, whereas FP32 accumulations that differ only in addition order agree within $2.4\times10^{-5}$. The accuracy comparison of Section~\ref{sec:accuracy_determinism} includes this rounding.

\section{Triangle Multiplication}
\label{tpfapp:trimul}
\label{sec:tpf-supp-trimul}

Triangle multiplication updates a pair by combining two transformed pair features that share a third token, contracting over that token separately for each channel. Our outgoing implementation (Figure~\ref{tpfapp:fig:trimul-pipeline}) places two fused CUDA kernels around one cuBLAS batched matrix multiplication. The input kernel normalizes each pair and computes five projections: two branch values, two branch gates, and the output-gate logits. The output kernel normalizes the contraction result, projects it, and applies the saved output gate. Fusion removes the global-memory round trips of intermediates within each stage, while the branch tensors and the contraction result still pass through global memory. The backward pass evaluates the projection tiles again instead of storing them and combines weight-gradient chunks in a fixed FP32 order; the attention replay results do not cover TriMul.

Unlike cuEquivariance 0.10.0, which evaluates the output gate in its output stage in FP32 registers, we compute the gate logits in the input kernel and store them in BF16. Against a PyTorch FP32 reference, the median relative $L_2$ errors nevertheless match closely: 0.386\% versus 0.384\% for the output and 0.604\% versus 0.607\% for the input gradient, compared with 0.440\% and 0.823\% for PyTorch with BF16 inputs. This compares complete implementations and does not isolate the effect of BF16 gate storage.

\begin{figure}[!t]
\centering
\includegraphics[width=\linewidth]{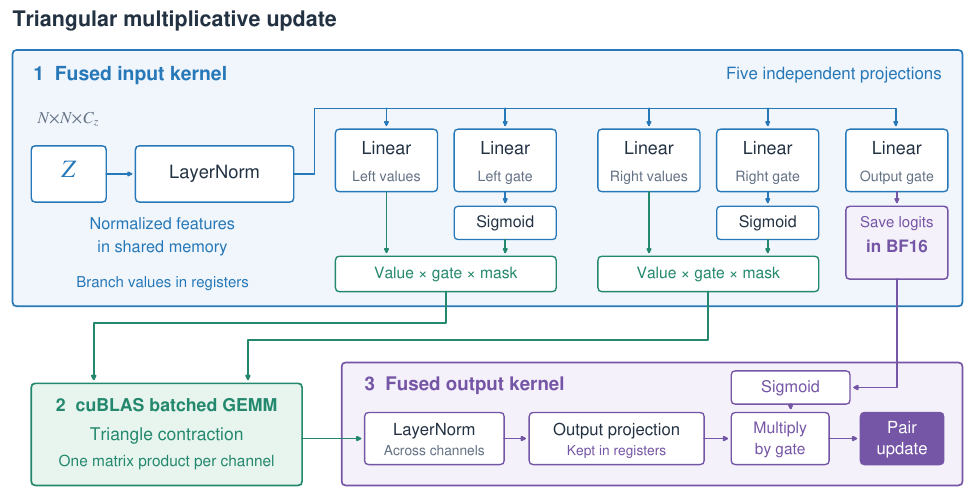}
\caption{\textbf{Outgoing triangle multiplicative update.} The fused input kernel normalizes the pair features and computes five independent projections: two branch values, two branch gates, and the output gate logits. The normalized features use shared memory, while branch values remain in registers until gating and masking. A separate cuBLAS batched GEMM contracts the branch operands independently for each channel. The fused output kernel combines normalization, output projection, and multiplication by the sigmoid of the saved output gate logits.}
\label{tpfapp:fig:trimul-pipeline}
\end{figure}

\paragraph{Component ablation.}
Table~\ref{tpfapp:tab:component-ablation} replaces attention, TriMul, or both in a cuEquivariance-enabled OpenFold3 run of 100 steps per crop with gradient accumulation of 16. The percentages are reductions in full training step latency, and the combined results are measured directly rather than summed.

\begin{table}[!htbp]
\centering
\caption{Training step latency reduction (\%) from replacing attention, TriMul, or both in a cuEquivariance-enabled OpenFold3 baseline.}
\label{tpfapp:tab:component-ablation}
\small
\begin{tabular}{@{}lrrr@{}}
\toprule
Crop & Attention only & TriMul only & Both \\
\midrule
384 & 2.21 & 2.86 & 5.55 \\
640 & 5.59 & 2.78 & 8.45 \\
768 & 6.43 & 2.99 & 9.63 \\
\bottomrule
\end{tabular}
\end{table}

\end{document}